%% file: main.tex
\documentclass[twocolumn]{scholar7}

\usepackage{amsmath,amssymb,amsthm,mathtools}

\usepackage[dvipsnames]{xcolor} 
\usepackage{tikz}
\usetikzlibrary{shapes.geometric, arrows.meta, positioning, fit, backgrounds, calc, shadows.blur, patterns}
\usepackage{pgfplots}
\usepgfplotslibrary{fillbetween}
\pgfplotsset{compat=1.18}
\usepackage{pgffor}

\usepackage{bm}
\usepackage{bbm}

\usepackage{fontawesome}

\definecolor{darkblue}{rgb}{0.0, 0.0, 0.55}

\usepackage{hyperref}
\hypersetup{
    colorlinks=true,
    linkcolor=darkblue,
    citecolor=darkblue,
    urlcolor=darkblue,
}

\newtheorem{theorem}{Theorem}
\newtheorem{lemma}[theorem]{Lemma}
\newtheorem{proposition}[theorem]{Proposition}
\newtheorem{corollary}[theorem]{Corollary}
\newtheorem{definition}[theorem]{Definition}
\newtheorem{assumption}[theorem]{Assumption}
\newtheorem{remark}[theorem]{Remark}

\newcommand{\cB}{\mathcal{B}}

\newcommand{\cT}{\mathcal{T}}
\newcommand{\cU}{\mathcal{U}}
\newcommand{\cV}{\mathcal{V}}

\newcommand{\cZ}{\mathcal{Z}}
\newcommand{\cX}{\mathcal{X}}
\newcommand{\cY}{\mathcal{Y}}

\newcommand{\E}{\mathbb{E}}
\newcommand{\Prob}{\mathbb{P}}
\newcommand{\KL}{\mathrm{KL}}
\newcommand{\I}{\mathrm{I}}
\newcommand{\Hh}{\mathrm{H}}

\newcommand{\abs}[1]{\left|#1\right|}
\newcommand{\bracks}[1]{\left[#1\right]}

\newcommand{\idx}[1]{\bracks{#1}} 

\usepackage{enumitem}

\title{\textsc{Epistemic Throughput}: Fundamental Limits of
Attention-Constrained Inference}

\author[1,2]{Lei You}

\affiliation[1]{Technical University of Denmark}
\affiliation[2]{CSPaper @ Scholar7}

\correspondence{\email{leiyo@dtu.dk}, \email{lei@scholar7.com}}
\date{February 2026}

\abstract{%
Recent generative and tool-using AI systems can surface a large volume of candidates at low marginal cost, yet only a small fraction can be checked carefully. This creates a decoder-side bottleneck: downstream decision-makers must form reliable posteriors from many public records under scarce attention. We formalize this regime via Attention-Constrained Inference (ACI), in which a cheap screening stage processes $K$ records and an expensive verification stage can follow up on at most $B$ of them. Under Bayes log-loss, we study the maximum achievable reduction in posterior uncertainty per window, which we call \emph{epistemic throughput}.
Our main result is a ``JaKoB'' scaling law showing that epistemic throughput has a baseline term that grows linearly with verification and prevalence, and an additional \emph{information-leverage} term that scales as $\sqrt{JKB}$, where $J$ summarizes screening quality. Thus, expanding cheap screening can nonlinearly amplify scarce verification, even when informative records are rare. We further show that this scaling is tight in a weak-screening limit, and that in the sparse-verification regime ($B \ll K$), substantial leverage requires heavy-tailed score distributions; for light-tailed scores the amplification is only logarithmic.
}

\begin{document}

\maketitle

\input{sections/introduction}
\input{sections/related_work}
\input{sections/system_model}
\input{sections/theory}
\input{sections/conclusion_and_discussion}


\input{sections/appendices}

\bibliographystyle{unsrtnat}
\bibliography{references}

\end{document}

%% file: sections/introduction.tex
\section{Introduction}
\label{sec:intro}

\subsection{Background and key problem}

Generative AI has fundamentally inverted the economics of information: the cost of generating plausible content has dropped to zero~\cite{galdin2025making}, while the cost of verifying its truth remains distinctively high.
This creates a crisis in \emph{high-stakes} domains, such as scientific discovery, software supply chains, and financial intelligence, where the system's objective is not engagement but \emph{inference}.
Unlike casual content consumption, these domains tolerate near-zero error; their goal is to reconstruct a binary latent state (truth, safety, or validity) from a noisy observable artifact.
Historically, trust in these systems relied on \emph{costly signaling}~\cite{spence1973}, where generation cost served as a proxy for quality.
LLMs have shattered this separating equilibrium, flooding the channel with cheap artifacts that mimic authority but lack the underlying guarantee of truth.

This disruption is catastrophic because verification does not scale.
There is no single ``editor-in-chief'' (or audit pipeline) with enough bandwidth to vet the global stream.
In practice, verification is carried out by downstream decision-makers. Sometimes there is only one, but often there are many, and each operates under scarce attention budgets.
We therefore face an \emph{attention-constrained inference} (ACI) problem.
Here, ``attention'' means a budget of costly inspection and verification steps, not the self-attention mechanism in Transformers.
Our focus is the maximum achievable information gain under log-loss per window, which we call \emph{epistemic throughput}.
To cope with informational overload, agents naturally adopt a \textbf{two-stage inference strategy}, consistent with rational inattention~\cite{sims2003rational} and visual search~\cite{treisman1980feature}, to maximize information gain about the latent ground truth under attention constraints:
\begin{enumerate}
\item \textbf{Screening ($K$)} (broad attention): inspect $K$ records using a low-cost, high-throughput heuristic (e.g., scanning an abstract). Screening is scalable but inherently noisy.
\item \textbf{Verification ($B$)} (deep attention): select a subset of $B$ candidates from the screened pool for high-fidelity auditing (e.g., reproducing an experiment). Verification is accurate but scarce.
\end{enumerate}

This abstraction matches modern retrieval augmented generation (RAG) workflows.
Screening corresponds to retrieval and lightweight scoring over a large pool, while verification corresponds to consuming and integrating a small set of sources well enough to justify a low-uncertainty output.

The central operational challenge therefore shifts from channel capacity to \emph{attention constraints}: how to use a noisy screen to aggressively filter the haystack so that the scarce verification budget is allocated only to the most promising signals.

\subsection{Main contributions}

\textbf{JaKoB scaling law.}
We characterize the best achievable \emph{information gain} under log-loss in the ACI haystack regime.
For a signal with prevalence $p$ and screening quality $J$, we prove a fundamental scaling law (formalized as a converse in Theorem~\ref{thm:ver-tradeoff} and shown achievable in a weak-screening limit in Theorem~\ref{thm:score-achievability}):
\begin{equation}
\label{eq:intro-jakob}
\mathrm{Gain}
\;\approx\;
I_{\mathrm{ver}}\,B\Bigg(
p
+
c\,\sqrt{\tfrac{J K}{B}}
\Bigg),
\qquad
c\leq \sqrt{\tfrac{\ln 2}{2}}.
\end{equation}

Here, $I_{\mathrm{ver}}$ is the information contributed per verified informative record (defined precisely in Section~\ref{sec:theory}).
This factorization can be read as the verification capacity $I_{\mathrm{ver}}B$ scaled by an \emph{enhanced hit rate}.
The square-root term reveals a tight and surprisingly universal mechanism:
even weak screening ($J$ small) can nonlinearly amplify a scarce verification budget, provided that screening can oversample at rate $K/B$.

\begin{itemize}
  \item \emph{Base precision ($p$):} the prevalence (sparsity) of informative records; this is the hit rate under random verification.
  \item \emph{Screening boost} ($c\sqrt{JK/B}$): the precision gain enabled by oversampling.
  By screening a large volume ($K$) to fill a small verification budget ($B$), the system can ``cherry-pick'' and dramatically enrich the verified subset.
\end{itemize}

Fig.~\ref{fig:JKB-scaling} offers a geometric view of the JaKoB tradeoff.
The gray \emph{Random} line is the gain from verifying blindly (hit rate $p$), and the gray \emph{Oracle} line is the ceiling under perfect screening (hit rate $1$).
The colored curves show the JaKoB prediction for different values of $JK$, interpolating between these extremes.
The background shading visualizes the yield per verified item.

\input{figures/JKB_scaling_law}

\textbf{Constructive achievability.}
Crucially, we show that the JaKoB limit is not merely a bound: it is efficiently achievable.
A simple \emph{score-based verification} policy (inspect $K$ records, rank them by screening scores, and verify the top-$B$) achieves the scaling in a practical \emph{weak-screening regime} (modeled via a local logistic regression).

\textbf{Escaping the Gaussian trap.}
In the sparse verification regime ($B\ll K$), the benefit of massive screening is controlled by the \emph{upper-tail behavior} of the screening scores.
For light-tailed scores (e.g., Gaussian), leverage grows only logarithmically; for heavy-tailed scores (e.g., Pareto), leverage can be polynomial (Propositions~\ref{prop:gaussian-tail}--\ref{prop:pareto-tail}).
To sharpen intuition, Appendix~\ref{app:tight-benchmark} also gives a fully solved benchmark with an exact finite-block characterization of the risk-resource boundary.

\subsection{Paper organization}

Section~\ref{sec:related} discusses related work.
Section~\ref{sec:model} presents the ACI model and the haystack specialization.
Section~\ref{sec:theory} states the main converse, achievability, and scaling results. Section~\ref{sec:conclusion} concludes the paper.
Appendix~\ref{sec:proofs} contains proofs, and Appendix~\ref{app:tight-benchmark} provides the fully solved benchmark.

%% file: figures/JKB_scaling_law.tex
\begin{figure}[t]
  \centering
  \begin{tikzpicture}
    \definecolor{acliBlue}{RGB}{31,119,180}
    \definecolor{acliPurple}{RGB}{155,89,182}
    \definecolor{acliGreen}{RGB}{44,160,44}
    \definecolor{acliGray}{RGB}{120,120,120}
    \definecolor{acliRed}{RGB}{214, 39, 40}

    \definecolor{effLowBlue}{RGB}{116, 173, 209} 
    \definecolor{effMid}{RGB}{250, 250, 245}
    \definecolor{effHighRed}{RGB}{244, 109, 67}  

    \begin{axis}[
        width=0.96\linewidth,
        height=0.90\linewidth,
        xmin=0, xmax=200,
        ymin=0, ymax=210,
        xlabel={Verification budget $B$},
        ylabel={Normalized gain $\mathrm{Gain}/I_{\mathrm{ver}}$},
        tick label style={font=\small},
        label style={font=\small},
        axis line style={line width=0.75pt},
        tick style={line width=0.75pt},
        grid=none,
        major grid style={line width=0.35pt, draw=black!7},
        minor tick num=1,
        axis on top,          
        view={0}{90},         
        title = {\textbf{JaKoB Scaling Law}}
      ]

      \pgfmathsetmacro{\baseP}{0.05} 

      \addplot3[
          surf,
          shader=interp,
          draw=none,
          domain=0.1:200,       
          domain y=\baseP:1,    
          samples=150,           
          samples y=40,         
          point meta=y/x,
          colormap={effmap_smooth}{
            color=(effLowBlue!60) 
            color=(effMid)       
            color=(effHighRed!60)
          },
      ] ({x}, {x*y}, {0});


      
      \addplot[densely dotted, very thick, color=acliGray] coordinates {(0,0) (200,200)};
      \node[font=\small, text=acliGray, anchor=west, rotate=43] at (axis cs:10,20) {Oracle: $B,~(\gamma=1)$};

      \addplot[dashed, very thick, color=acliGray] coordinates {(0,0) (200,{200*\baseP})};
      \node[font=\small, text=acliGray, anchor=west, rotate=5] at (axis cs:105,8.5) {Random: $Bp,~(\gamma=p)$};

      \pgfmathsetmacro{\constC}{0.45}    

      \pgfmathsetmacro{\JKtwo}{10}
      \pgfmathsetmacro{\BstarA}{(\constC * sqrt(\JKtwo) / (1-\baseP))^2}
      
      \addplot[solid, very thick, color=acliPurple, domain=0:\BstarA] {x};
      \addplot[solid, very thick, color=acliPurple, domain=\BstarA:200, samples=60]
        {\baseP*x + \constC*sqrt(\JKtwo*x)};
        
      \node[font=\small, text=acliPurple, anchor=west, rotate=7] at (axis cs:140,30) {$J\!\cdot\!K=10$};
      \addplot[acliPurple, densely dashed, opacity=0.55, line width=1.0pt] coordinates {(\BstarA,0) (\BstarA,\BstarA)};
      \addplot[only marks, mark=*, mark size=1.6pt, color=acliPurple] coordinates {(\BstarA,\BstarA)};

      \pgfmathsetmacro{\JKmain}{100}
      \pgfmathsetmacro{\BstarB}{(\constC * sqrt(\JKmain) / (1-\baseP))^2}

      \addplot[solid, very thick, color=acliBlue, domain=0:\BstarB] {x};
      \addplot[solid, very thick, color=acliBlue, domain=\BstarB:200, samples=60]
        {\baseP*x + \constC*sqrt(\JKmain*x)};
        
      \node[font=\small, text=acliBlue, anchor=west, rotate=7] at (axis cs:140,70) {$J\!\cdot\!K=100$};
      \addplot[acliBlue, densely dashed, opacity=0.55, line width=1.0pt] coordinates {(\BstarB,0) (\BstarB,\BstarB)};
      \addplot[only marks, mark=*, mark size=1.6pt, color=acliBlue] coordinates {(\BstarB,\BstarB)};

      \pgfmathsetmacro{\JKten}{200}
      \pgfmathsetmacro{\BstarC}{(\constC * sqrt(\JKten) / (1-\baseP))^2}

      \addplot[solid, very thick, color=acliGreen, domain=0:\BstarC] {x};
      \addplot[solid, very thick, color=acliGreen, domain=\BstarC:200, samples=60]
        {\baseP*x + \constC*sqrt(\JKten*x)};
        
      \node[font=\small, text=acliGreen, anchor=west, rotate=7] at (axis cs:140,96) {$J\!\cdot\!K=200$};
      \addplot[acliGreen, densely dashed, opacity=0.55, line width=1.0pt] coordinates {(\BstarC,0) (\BstarC,\BstarC)};
      \addplot[only marks, mark=*, mark size=1.6pt, color=acliGreen] coordinates {(\BstarC,\BstarC)};

      \pgfmathsetmacro{\JKten}{400}
      \pgfmathsetmacro{\BstarC}{(\constC * sqrt(\JKten) / (1-\baseP))^2}

      \addplot[solid, very thick, color=acliRed, domain=0:\BstarC] {x};
      \addplot[solid, very thick, color=acliRed, domain=\BstarC:200, samples=60]
        {\baseP*x + \constC*sqrt(\JKten*x)};
        
      \node[font=\small, text=acliRed, anchor=west, rotate=7] at (axis cs:140,133) {$J\!\cdot\!K=400$};
      \addplot[acliRed, densely dashed, opacity=0.55, line width=1.0pt] coordinates {(\BstarC,0) (\BstarC,\BstarC)};
      \addplot[only marks, mark=*, mark size=1.6pt, color=acliRed] coordinates {(\BstarC,\BstarC)};

      \node[
        anchor=north west, 
        align=left, font=\small, inner sep=5pt
      ] at (axis cs:4,206)
      {    $\displaystyle \frac{\mathrm{Gain}}{I_{\mathrm{ver}}}\approx
           \min\!\left\{B,\;B\cdot \left(p+c\sqrt{\frac{JK}{B}}\right)\right\}$
        };

    \node[font=\small, anchor=south east] at (axis cs:126,125)
      {Yield per unit $B$: $\displaystyle \gamma=\frac{\mathrm{Gain}}{B I_{\mathrm{ver}}}$};
    \node[font=\footnotesize, anchor=south east] at (axis cs:110,110)
      {\textcolor{effLowBlue!80!black}{\textbf{Low $\gamma=p$}} $\rightarrow$ \textcolor{effHighRed!80!black}{\textbf{High $\gamma=1$}}};

    \end{axis}
  \end{tikzpicture}
\caption{\textbf{JaKoB scaling and information leverage.}
Blind verification yields the baseline $Bp$, while perfect screening caps at the oracle ceiling $B$.
With screening quality $J$ applied to $K$ candidates, screening contributes an additional $\Theta(\sqrt{JKB})$ gain (clipped at $B$); larger $JK$ shifts the curve upward, so the same information gain can be reached with a smaller scarce budget $B$ and a higher per-verification yield $\gamma=\mathrm{Gain}/(B I_{\mathrm{ver}})$.}
  \label{fig:JKB-scaling}
\end{figure}
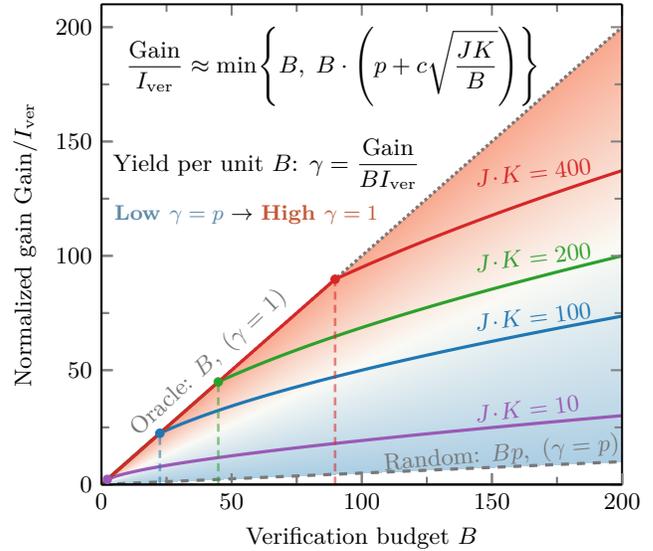

%% file: sections/related_work.tex
\section{Related Work}
\label{sec:related}

Our model sits at the intersection of information theory, decision theory, and resource-constrained inference.
Classical information theory focuses on \emph{encoder-side} constraints, including rate-limited descriptions and noisy channels, as in \citep{shannon1948} and standard texts \citep{cover2006elements}.
ACI instead isolates a complementary regime in which \emph{records are abundant and public}, but \emph{downstream decision-makers are attention-limited}:
only a small subset can be inspected and an even smaller subset can be verified.
Below we organize the most relevant threads and clarify where ACI differs.

\subsection{Decoder-side constraints and information acquisition}

A large literature studies how resource constraints at the receiver affect communication and inference.
In information theory, one line constrains \emph{decoding resources} such as reliability, memory, or energy, and asks how these constraints modify fundamental limits; see, e.g., the thesis of \citep{varshney2010unreliable} and related work on faulty or resource-limited iterative decoding.
A second, closely related line treats \emph{information acquisition} as an \emph{action} under a cost constraint:
the system chooses whether/how to obtain side information.
This is formalized in ``source coding with actions'' and ``vending machine'' models, where actions control the quality or availability of side information \citep{permuter2011vending}.
ACI shares the acquisition viewpoint but differs operationally:
our two-stage structure separates \emph{cheap screening} (inspection) from \emph{expensive high-fidelity observation} (verification),
and the key coupling parameter is the screening quality (quantified by the mutual information between the screening signal and the latent truth), which governs how much screening can enrich the verified set.

Economics provides a complementary and influential formalization of limited attention via \emph{rational inattention},
where information-processing costs are measured through Shannon mutual information \citep{sims2003rational}.
This perspective has been developed into tractable choice models \citep{matejka2015rational,caplin2015revealed} and surveyed in \citep{mackowiak2023review}.
Our model is aligned with the mutual-information cost philosophy, but with an engineering emphasis:
we ask how a community-level inspection budget $K$ can amplify a verification budget $B$ in terms of Bayes log-loss.

Finally, our results relate to classic \emph{sequential design} and \emph{active learning} themes
(e.g., choosing what to measure next) dating back to \citep{chernoff1959}.
ACI is deliberately \emph{one-shot} at the window level:
screening produces a batch of scores and verification selects a subset (often top-$B$),
making the analysis closer to selection under budget than to fully adaptive querying.

\subsection{Decentralized inference and social learning}

Distributed detection and estimation traditionally study networks of sensors that send messages to a fusion center or cooperate under communication constraints.
Foundational results include decentralized detection with many sensors \citep{tsitsiklis1988decentralized} and the detection/data-fusion synthesis in \citep{varshney1997distributed}.
Modern surveys emphasize resource constraints and network structure in distributed inference \citep{veeravalli2012distributed}.
These models typically assume a \emph{designed sensing architecture} and a \emph{communication protocol}.
ACI takes a different stance:
records already exist as public artifacts, and the bottleneck lies in what decoders can afford to \emph{inspect} and \emph{verify}.

Although ACI does not impose a multi-agent coordination constraint and applies even in the single-decoder case, its motivation of many downstream agents reusing public artifacts connects to social learning and belief aggregation.
Early models of consensus and belief pooling include \citep{degroot1974}.
In economics, herd behavior and informational cascades \citep{banerjee1992,bikhchandani1992} show how public actions can dominate private information.
More recent networked-learning work characterizes when agents learn from neighbors without a central coordinator \citep{jadbabaie2012nonbayesian}.
ACI is compatible with these perspectives but focuses on a different bottleneck:
we model the \emph{production of verifiable public artifacts} under shared budgets, and evaluate the best achievable population-level log-loss.

\subsection{Log-loss as an operational metric}

Log-loss is central in information theory because it turns Bayes risk into conditional entropy.
Beyond its role in coding under logarithmic loss \citep{courtade2014multiterminal},
log-loss is the canonical scoring rule for probabilistic prediction and is deeply tied to universal coding and Bayesian mixture methods \citep{clarke1990bayes,merhav1998universal,rissanen1984universal}.
Recent work continues to study prediction under log-loss with side information and regret characterizations \citep{bhatt2021sequential}.
In ACI, log-loss provides a clean risk metric that naturally aligns with the objective of maximizing information gain, making the inspection and verification interaction analytically transparent.

\subsection{Haystacks, rare signals, and extreme-value selection}

The haystack regime, where there are many candidates but only a small fraction of items are informative, is shared with statistics on sparse or rare signal detection and multiple testing.
Work on detecting sparse mixtures and rare effects highlights sharp phase transitions and the need to aggregate weak evidence \citep{donoho2004higher,ingster2003testing}.
Verification as ``testing a few'' among many hypotheses also relates to multiple-comparisons ideas such as false discovery rate control \citep{benjamini1995fdr}.
In settings where tests can be pooled, group testing \citep{dorfman1943group} provides another classic abstraction of scarce testing resources.
ACI differs in that verification reveals high-fidelity information \emph{only when the record is informative},
so the central object becomes \emph{enrichment} of the verified set from cheap screening.

A complementary, information-theoretic view of selection appears in adaptive data analysis.
Russo and Zou \citep{russo2019much} bound the bias induced by data-dependent exploration by the mutual information between the selected analysis and the data, including explicit treatments of filtering and rank selection.
Although their objective is to control post-selection bias rather than maximize information gain, the underlying information-usage principle is closely related to our enrichment converse: limited screening information constrains how far a top-$B$ rule can tilt the verified set away from the baseline prevalence.

Technically, the top-$B$ selection step makes order statistics and extreme-value theory unavoidable \citep{david2003order,embrechts1997modelling}.
Our tail-leverage results formalize exactly when expanding inspection $K$ is valuable: it is valuable only to the extent that the score distribution has exploitable upper-tail mass.

\subsection{Verification at scale in modern information ecosystems}

Finally, ACI is motivated by empirical systems where verification is scarce relative to content generation.
In NLP and fact checking, large-scale claim verification datasets and benchmarks (e.g., FEVER) formalize the pipeline of retrieving candidates and verifying a small set \citep{thorne2018fever}.
In the foundation-model era, community reports and analyses emphasize the gap between cheap content generation and expensive auditing \citep{bommasani2021foundation,bender2021parrots,openai2023gpt4}.
ACI is not a proposal for a new verifier; rather, it provides \emph{limits} that constrain \emph{any} such pipeline under inspection/verification budgets.

A common mitigation is to ground generation in external evidence through RAG~\cite{lewis2020rag} and related retrieval-based pipelines.
Retrieval can cheaply surface many plausible sources, but only a few can be read, cross-checked, and trusted within a limited context window or compute budget.
The bottleneck remains verification attention.

The same bottleneck appears in tool-using agents that can run many low-cost queries or heuristics but can only follow up deeply on a small set of results. Representative instances include web-browsing QA agents such as WebGPT \cite{nakano2021webgpt},
ReAct-style agents that interleave reasoning with tool calls \cite{yao2023react},
Toolformer-style self-supervised tool use \cite{schick2023toolformer}
and modular tool-routing architectures such as MRKL systems \cite{karpas2022mrkl}. A related prompting approach is self-ask, which can be paired with a search engine to answer follow-up questions \cite{press2023measuring}. Another relevant application is CSPaper Review that selects the most justified critiques from a pool of concurrent agent-generated reviews \cite{cao2025cspaper}.

%% file: sections/system_model.tex
\section{System Model}
\label{sec:model}

This section introduces our one-window model for attention-constrained inference (ACI) and the performance metric under log-loss.
A population of producers releases a large collection of records, while a population of decoders has limited capacity to inspect them.
Later in the section we specialize to a canonical \emph{haystack} regime that captures a two-stage pipeline, cheap screening followed by expensive verification, which is the setting analyzed in Section~\ref{sec:theory}.
All random variables are defined on a common probability space.

Fig.~\ref{fig:system_model} previews the haystack specialization.
In one window, the system can screen many candidates at low cost and then follow up deeply on only a small subset.
We formalize both the general model and this specialization below.

\input{figures/system_model}

\subsection{Producers, records, and decoders}
Let $\Theta$ be a latent state taking values in a finite set $\cT$, with prior $P_\Theta$.
There are $m$ producers indexed by $i\in\idx{m}$.
Producer $i$ observes evidence $E_i\in\cY$ generated according to a conditional distribution $P_{E_i\mid\Theta}$, and publishes a public record
\[
  X_i=f_i(E_i)\in\cX.
\]
We write $X^m=(X_1,\dots,X_m)$ for the entire collection of records released in a window.

There are $N$ decoders indexed by $j\in\idx{N}$.
We allow multiple decoders to model community settings where verified artifacts are public and can be reused.
At the same time, our fundamental limits depend only on aggregate budgets, so the results apply equally to the single-decoder case $N=1$.
Decoder $j$ starts with a prior belief $\pi_j$ on $\Theta$ (capturing side information that may differ across decoders) and forms an estimate after inspecting only a small subset of the records.

\subsection{Attention constraints and routing}
Decoder $j$ has an attention budget $k_j\in\mathbb{Z}_{\ge 0}$.
In one window it inspects a subset
\[
  S_j\subseteq\idx{m},\qquad \abs{S_j}\le k_j,
\]
and observes the corresponding inspected information set $I_j=\{X_i:i\in S_j\}$.
We define the total attention capacity
\[
  K:=\sum_{j=1}^N k_j.
\]

Many systems expose cheap public metadata for every record (timestamps, user features, embeddings, coarse heuristic scores, and so on).
We model this as a public signal $U_i\in\cU$ that is visible to all decoders before they choose what to inspect.
A routing rule maps the collection of public signals into inspected sets.
Formally, decoder $j$ selects
\[
  S_j=r_j(U^m,\pi_j,\omega_j),
\]
through routing $r_j$, with $\omega_j$ being private randomness.
We assume that all private randomness variables are independent of the record-generation process, including $\Theta$ and all producer-side observations.

\subsection{Log-loss risk}
After observing $I_j$ and starting from prior $\pi_j$, decoder $j$ outputs a distribution $q_j$ on $\cT$.
Under log-loss, the incurred loss is $-\log q_j(\Theta)$.
We evaluate population-level performance by the average Bayes risk
\begin{equation}
  \mathrm{Risk}:=\frac{1}{N}\sum_{j=1}^N \E\bracks{-\log q_j(\Theta)}.
  \label{eq:risk-pop}
\end{equation}
When $q_j$ matches the true Bayesian posterior, the risk attains its theoretical lower bound, $\Hh(\Theta\mid I_j,\pi_j)$. This identity establishes that \emph{the fundamental limit of inference is governed by the residual uncertainty} in the data, thereby justifying the use of Information Gain (entropy reduction) as the canonical metric for epistemic throughput.

\subsection{Haystack specialization: screening and verification}
The main results in Section~\ref{sec:theory} focus on a canonical haystack regime in which screening is cheap and verification is expensive.
The key modeling ingredient is that only a small fraction of inspected records can become informative about $\Theta$ upon verification.
We capture this sparsity using a latent type indicator.

\begin{assumption}[Informative records and screening]
\label{asmp:screening}
Each inspected record has a latent type $T\in\{0,1\}$ with $\Prob(T=1)=p$, where $p\in(0,1)$.
An inspection reveals a screening statistic $Z\in\cZ$.
The pair $(T,Z)$ is independent of $\Theta$.
We measure screening quality by the mutual information $J:=\I(T;Z)$.
\end{assumption}

Assumption~\ref{asmp:screening} says that screening does not directly reveal information about the target $\Theta$.
Instead, it provides a noisy proxy for whether a record is worth verifying.
Policies may map the screening statistic to a score $\eta(Z)$ and use it to decide which records to verify (for example, the Bayes score $\eta(z)=\Prob(T=1\mid Z=z)$ used in our achievability analysis later).

Verification is the expensive channel that can reduce uncertainty about $\Theta$.
In the canonical model we analyze, verification reveals the type together with an additional observation.
Revealing $T$ is a modeling simplification that isolates the role of verification as a distinct side channel.

\begin{assumption}[Verification channel]
\label{asmp:verification-channel}
If a record is verified, the decoder observes a pair $(T,V)$.
Conditioned on $(\Theta,T)$, the variable $V$ is independent of $Z$.
If $T=0$, then $V$ is independent of $\Theta$.
If $T=1$, then $V$ carries information about $\Theta$.
We define
\begin{equation}
  I_{\mathrm{ver}}:=\I(\Theta;V\mid T=1).
  \label{eq:iver}
\end{equation}
\end{assumption}

A policy in the haystack regime screens $K$ candidate records in a window and verifies at most $B$ of them.
We index the screened candidates by $i\in\idx{K}$ and write $\cB\subseteq\idx{K}$ for the (random) verified index set, with $\abs{\cB}\le B$.
The verified outputs $\{(T_i,V_i)\}_{i\in\cB}$ are published as public artifacts and can be reused by all decoders.

We include two specializations of the latent state $\Theta$ to cover different coupling structures across records.
In the global-$\Theta$ specialization, all $K$ records share a single target $\Theta$; conditioned on $\Theta$, the tuples $(T_i,Z_i,V_i)$ are i.i.d.\ across $i$, but they are generally correlated after marginalizing over $\Theta$.
In the per-record-claim specialization, $\Theta=(\Theta_1,\ldots,\Theta_K)$ collects independent claim variables; conditioned on $(\Theta_1,\ldots,\Theta_K)$ the tuples are independent across $i$ (not necessarily identically distributed given $\Theta$), while marginally the tuples are i.i.d.\ across $i$.
In both specializations, conditioned on $(\Theta,T_i)$ in the global model and on $(\Theta_i,T_i)$ in the decoupled model, the verification output $V_i$ is independent of all screening statistics and all other verification outputs.

\textbf{Scope of results.}
Our converse results (notably Lemma~\ref{lem:enrichment} and Theorem~\ref{thm:ver-tradeoff}) apply to both specializations, as they depend only on the screening/verification channels, budgets $(K,B)$, and information parameters $(J,I_{\mathrm{ver}})$.
Our achievability and tightness results are proved under the decoupled-claim specialization (Assumption~\ref{asmp:decoupled-claims}), where the verification stage reduces to an optimal selection problem without global coupling.

We interpret $(K,B)$ as community-level budgets.
This is a natural abstraction when verified artifacts can be published as public objects that all decoders can use.

\begin{assumption}[Public artifacts and shared access]
\label{asmp:public-artifacts}
Within one window, every verification output is published as a public artifact.
All decoders can access the set of verified artifacts.
Therefore, for fundamental-limit analysis under budgets $(K,B)$, we can evaluate performance using a hypothetical agent that observes all verification outputs generated in the window.
This does not introduce a fusion center into the system.
It is only an evaluation device that matches what any decoder can reconstruct from public artifacts.
\end{assumption}

%% file: figures/system_model.tex
\begin{figure*}[t]

\definecolor{paperBlue}{RGB}{55, 110, 180}   
\definecolor{paperTeal}{RGB}{0, 150, 136}    
\definecolor{paperRed}{RGB}{230, 80, 80}     
\definecolor{paperGray}{RGB}{236, 240, 241}  
\definecolor{paperDark}{RGB}{44, 62, 80}     
\definecolor{paperLight}{RGB}{189, 195, 199} 

\centering
\resizebox{\textwidth}{!}{
\begin{tikzpicture}[
    font=\small,
    >=LaTeX,
    node distance=1.2cm, 
    block/.style={
        draw=paperDark!80,
        thick,
        rounded corners=3pt, 
        fill=white,
        text=paperDark,
        align=center,
        blur shadow={shadow blur steps=5, shadow xshift=1pt, shadow yshift=-1pt}
    },
    process/.style={
        block,
        fill=paperBlue!8, 
        draw=paperBlue!80,
        minimum height=2.8cm,
        minimum width=2.4cm
    },
    verify/.style={
        block,
        fill=paperTeal!8,
        draw=paperTeal!80,
        minimum height=2.8cm,
        minimum width=2.4cm
    },
    doc/.style={
        draw=paperDark!40,
        fill=paperLight!30,
        minimum width=0.45cm,
        minimum height=0.55cm,
        inner sep=0pt
    },
    docTrue/.style={
        doc,
        fill=paperRed!80,
        draw=paperRed!60!black
    },
    arrow/.style={
        ->,
        draw=paperDark!80,
        thick,
        line width=1pt
    },
    label/.style={
        text=paperDark,
        font=\bfseries\footnotesize
    },
    sublabel/.style={
        text=paperDark!70,
        font=\scriptsize
    }
]

    \node[label, align=center] (input_label) at (0, 2.0) {Public Records \\ (The Haystack)};
    
    \foreach \x in {-0.6, 0, 0.6} {
        \foreach \y in {-0.6, 0, 0.6, 1.2} {
             \node[doc] at (\x, \y) {}; 
        }
    }
    \node[docTrue] at (-0.6, 1.2) {};
    \node[docTrue] at (0, -0.6) {};
    \node[docTrue] at (0.6, 0.6) {};
    
    \node[sublabel] at (0, -1.2) {Total Volume: $K$};
    \node[fit={(-1,-1.2) (1,2)}, name=haystack_group] {};

    \node[process, right=2.0cm of haystack_group] (screening) {
        \textbf{Screening}\\
        (Broad Attention)\\[0.5em]
        \rule{1.5cm}{0.4pt}\\[0.5em]
        \textit{Noisy Filter}\\
        $Z \sim P_{Z|T}$
    };
    
    \draw[arrow] (0.9,0.4) -- node[above, font=\scriptsize] {Inspect All} node[below, font=\scriptsize] {$K$} (screening);

    \node[right=2.0cm of screening] (ranking_center) {};
    
    \begin{scope}[shift={(ranking_center)}]
        \node[label, above=1.6cm] {Sorted by Score $\eta(Z)$};
        
        \node[docTrue] (top1) at (0, 1.35) {};
        \node[docTrue] (top2) at (0, 0.85) {};
        \node[doc, fill=paperLight!60] (top3) at (0, 0.35) {}; 
        \node[docTrue] (top4) at (0, -0.15) {};
        
        \draw[densely dashed, paperRed, thick] (-0.8, -0.45) -- (0.8, -0.45);
        \node[right, font=\tiny, text=paperRed] at (0.8, -0.45) {Cutoff};
        
        \node[doc, opacity=0.6] at (0, -0.8) {};
        \node[doc, opacity=0.4] at (0, -1.3) {};
        
        \draw[->, dashed, draw=paperLight!80, thick, bend left=20] (0.3, -1.0) to[out=-30, in=180] (1.0, -2.0) node[right, font=\scriptsize, text=paperLight, rotate=45] {Discarded ($K-B$)};
    \end{scope}
    
    \node[fit={(ranking_center) (0,1.5) (0,-1.5)}, name=rank_group] {};

    \draw[arrow] (screening) -- node[midway, above, font=\scriptsize, align=center, yshift=2pt, rotate=0] {Scoring} (7.8,0.4);

    \node[verify, right=2.5cm of rank_group] (verification) {
        \textbf{Verification}\\
        (Deep Attention)\\[0.5em]
        \rule{1.5cm}{0.4pt}\\[0.5em]
        \textit{High Fidelity}\\
        $V \sim P_{V\mid \Theta, T}$

    };

    \draw[decorate, decoration={brace, amplitude=4pt, mirror}, thick, draw=paperDark] 
        ($(top4.south east)+(0.1,0)$) -- ($(top1.north east)+(0.1,0)$) 
        coordinate[midway, xshift=4pt] (brace_mid);
        
    \draw[arrow] (brace_mid) -- node[above, font=\scriptsize] {Select Top-$B$} (verification.west |- brace_mid);

    \node[block, right=1.2cm of verification, align=left, minimum width=2.2cm, fill=paperGray!20] (output) {
        \textbf{Public Artifacts}\\[0.3em]
        \scriptsize $\{(T_i, V_i)\}_{i \in \cB}$\\[0.2em]
        \scriptsize Gain $\propto \sqrt{JKB}$
    };

    \draw[arrow] (verification) -- node[above, font=\scriptsize] {Verify} (output);

    
    \node[below=0.15cm of screening, text=paperBlue, font=\scriptsize] (cap1) {Capacity: $J = \I(T;Z)$};
    \node[below=0.15 of verification, text=paperTeal, font=\scriptsize]  {Capacity: $B\cdot I_{\mathrm{ver}}$};

    \node[font=\scriptsize\bfseries, text=paperDark!60, anchor=north] at ($(screening.south)!0.5!(verification.south) + (-1, -0.4)$) {Leverage Factor: Oversampling Ratio $K/B$};

    \begin{scope}[on background layer]
        \node[fit=(haystack_group)(output)(rank_group)(cap1), fill=paperGray!40, rounded corners=8pt, draw=paperGray!100, dashed, inner sep=12pt] (bg) {};
        \node[anchor=north west, text=paperDark!50, font=\bfseries\scriptsize, inner sep=5pt] at (14,2.6) {ACI Inference Pipeline};
    \end{scope}

\end{tikzpicture}
}
\caption{\textbf{The ACI Inference Pipeline.}
The system operates as a two-stage information refinery in the haystack regime ($B \ll K$).
A massive volume of public records ($K$) is first filtered by a cheap, noisy screening channel ($Z$) to prioritize attention.
Based on the screening scores $\eta(Z)$, only the most promising top-$B$ candidates receive expensive, high-fidelity verification ($V$).
The \emph{oversampling ratio} $K/B$ acts as a leverage factor, amplifying the yield of the scarce verification budget to achieve the $\Theta(\sqrt{JKB})$ scaling.}
\label{fig:system_model}
\end{figure*}
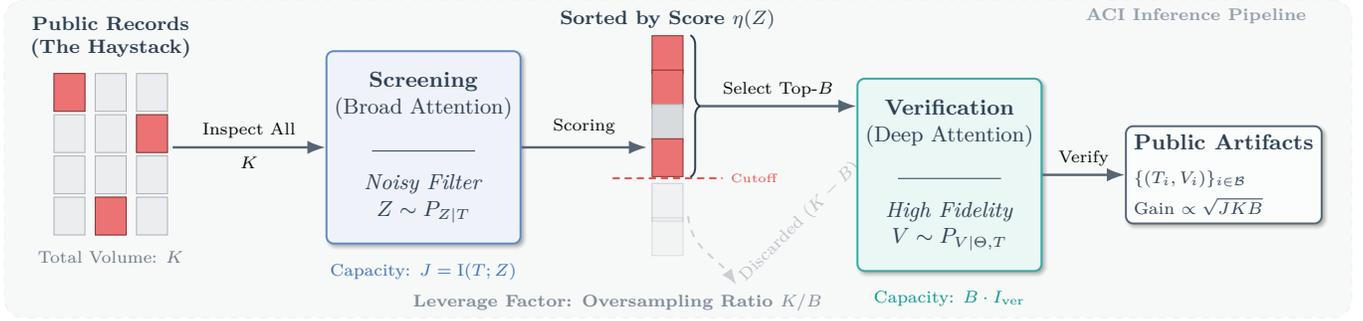

%% file: sections/theory.tex
\section{Main Theoretical Results}
\label{sec:theory}

This section studies the haystack specialization introduced in Section~\ref{sec:model}.
We focus on the non-adversarial setting.
Our goal is to quantify how inspection and verification interact under Bayes log-loss.

We use the following budgets.
The community inspects $K$ independent records and verifies at most $B$ among those inspected records.
We write $\alpha:=B/K$.
All logarithms are base $2$.
We use $\ln$ only for natural logarithms in log-odds expressions.
Entropy and mutual information are measured in bits.

\subsection{A fundamental tradeoff between inspection and verification}
\label{subsec:tradeoff}

Our first question is: given budgets $(K,B)$ and screening quality $J=\I(T;Z)$, how much can inspection amplify a scarce verification budget under log-loss?
Our main answer is Theorem~\ref{thm:ver-tradeoff}, which gives a universal converse bound on the information gain.
The key technical ingredient is a selection enrichment bound (Lemma~\ref{lem:enrichment}), which limits how much \emph{any} selection rule based on $Z$ can increase the hit rate of informative records.

\begin{lemma}[Selection enrichment bound]
\label{lem:enrichment}
Assume Assumption~\ref{asmp:screening}.
Let $S\in\{0,1\}$ be any (possibly randomized) selection rule that depends on $Z$ and on auxiliary randomness independent of $(T,Z)$, and let
$\alpha:=\Prob(S=1)$.
Then
\begin{equation}
\label{eq:enrichment}
\Prob(T=1\mid S=1)
\;\le\;
p + \sqrt{\frac{\ln 2}{2\alpha}\,J}.
\end{equation}
\end{lemma}

Lemma~\ref{lem:enrichment} makes the role of $J$ explicit.
It states that a cheap screening score cannot create an arbitrarily clean verified set.
The upper bound depends on $\alpha$ because enrichment is harder when verification is extremely sparse.

We now bound the information gain under budgets $(K,B)$.

\begin{definition}[Budgets $(K,B)$]
\label{def:budgets}
A policy uses budgets $(K,B)$ if it inspects $K$ i.i.d. records, and it verifies at most $B$ among the inspected records.
The policy may use the screening statistics $Z_1,\dots,Z_K$ to select which records to verify.
\end{definition}

Assumption~\ref{asmp:public-artifacts} lets us evaluate the community through the public artifacts created in one window.
For $i\in\idx{K}$, define the verification transcript
\[
  Y_i:=\begin{cases}
  \bot, & \text{if record $i$ is not verified},\\
  (T_i,V_i), & \text{if record $i$ is verified},
  \end{cases}
\]
where $\bot$ is a fixed symbol.
We write $\bm{Z}:=(Z_1,\dots,Z_K)$ and $\bm{Y}:=(Y_1,\dots,Y_K)$.
We define $D(K,B)$ as the Bayes-optimal expected log-loss after observing $(\bm{Z},\bm{Y})$ under budgets $(K,B)$.

\begin{theorem}[A tradeoff between verification and attention under log-loss]
\label{thm:ver-tradeoff}
Assume Assumptions~\ref{asmp:screening} and \ref{asmp:verification-channel}.
Consider any policy that uses budgets $(K,B)$ in the sense of Definition~\ref{def:budgets}.
Then
\begin{equation}
\label{eq:ver-tradeoff}
\Hh(\Theta)-D(K,B)
\;\le\;
B\cdot I_{\mathrm{ver}}\Bigg(
p + \sqrt{\frac{\ln 2}{2}}\,\sqrt{\frac{JK}{B}}
\Bigg).
\end{equation}
\end{theorem}

Theorem~\ref{thm:ver-tradeoff} is the main converse bound.
The linear term $Bp$ is the baseline gain of random verification.
The square-root term quantifies the best possible amplification from inspection.

The next corollary rewrites the converse as a lower bound on the verification budget required for a target gain.

\begin{corollary}[Verification required for a target gain]
\label{cor:ver-budget}
Assume the setting of Theorem~\ref{thm:ver-tradeoff}.
Fix any target information gain $\Delta\in(0,\Hh(\Theta))$.
If a policy satisfies $\Hh(\Theta)-D(K,B)\ge \Delta$, then $B$ must satisfy
\begin{equation}
\label{eq:ver-budget}
B \;\ge\;
\frac{1}{4p^2}
\Bigg(
\sqrt{\frac{\ln 2}{2}\,J K + \frac{4p\Delta}{I_{\mathrm{ver}}}}
-
\sqrt{\frac{\ln 2}{2}\,J K}
\Bigg)^2.
\end{equation}
\end{corollary}

Corollary~\ref{cor:ver-budget} separates two regimes.
If $J=0$, screening is useless and verification must scale as $\Omega(1/p)$ in the sparse limit.
If $J K$ is large, inspection can reduce the required verification budget.

\subsection{Achievability with score-based verification}
\label{subsec:achievability}

Theorem~\ref{thm:ver-tradeoff} exhibits a square-root \emph{amplification} term of order $\sqrt{J K/B}$.
A natural question is whether this term is achievable, or whether it is an artifact of the converse proof.
We answer in the affirmative: in a local weak-screening model (Assumption~\ref{asmp:weak-screening}), the simple top-$B$ score-based verification policy achieves the same scaling up to constants (Theorem~\ref{thm:score-achievability}).

Let
\[
\eta(z):=\Prob(T=1\mid Z=z)
\]
denote the Bayes score of the screening statistic.
We study a local regime in which $\eta(Z)$ is close to $p$.

\begin{assumption}[Weak screening through a log-odds score]
\label{asmp:weak-screening}
Let $\eta(Z)=\Prob(T=1\mid Z)$.
There exist a scalar $\varepsilon\ge 0$ and a real-valued random variable $G$ such that
\begin{equation}
\label{eq:logit-local}
\ln\frac{\eta(Z)}{1-\eta(Z)}
=
\ln\frac{p}{1-p}+\varepsilon G
\end{equation}
almost surely.
We assume $\E[G]=0$, $\E[G^2]=1$, and $\E[\abs{G}^3]<\infty$.
We also assume that $G$ has a continuous distribution.
\end{assumption}

Assumption~\ref{asmp:weak-screening} follows a standard Bayesian view of evidence accumulation rather than an ad hoc modeling choice.
On the log-odds scale, Bayes' rule is additive: the posterior log-odds equal the prior log-odds plus a weight-of-evidence term (a log-likelihood ratio), a perspective emphasized by Good~\cite{good1950probability} in cryptanalysis (with evidence measured in ``bans'').
Assumption~\ref{asmp:weak-screening} adopts a local parametrization by writing this log-odds increment as $\varepsilon G$.
The qualifier \emph{weak} refers to the regime $\varepsilon\to 0$, where $\eta(Z)=p+O(\varepsilon)$ and the screening information $J=\I(T;Z)$ vanishes accordingly (typically $J=O(\varepsilon^2)$).
We normalize $G$ to have mean zero and unit variance so that $\varepsilon$ captures the overall screening strength; the remaining moment and continuity conditions are regularity assumptions used to control the asymptotics.
From a modern machine-learning standpoint, \eqref{eq:logit-local} is the canonical logit-link form, closely related to logistic regression with a linear predictor~\cite{cox1958regression}, where $G$ is a standardized score and $\varepsilon$ sets its signal-to-noise level.

Fix $\alpha\in(0,1)$.
Let $q_\alpha$ be the $(1-\alpha)$-quantile of $G$, namely
\[
q_\alpha:=\inf\{q\in\mathbb{R}: \Prob(G\le q)\ge 1-\alpha\}.
\]
We define the upper-tail mean
\begin{equation}
\label{eq:tail-mean}
m_G(\alpha):=\frac{1}{\alpha}\,\E\!\bracks{G\,\1\{G\ge q_\alpha\}}.
\end{equation}

\begin{assumption}[Decoupled claims for achievability]
\label{asmp:decoupled-claims}
We use the following decoupled claim model in Theorem~\ref{thm:score-achievability}.
The latent state is $\Theta=(\Theta_1,\dots,\Theta_K)$, where $\Theta_1,\dots,\Theta_K$ are i.i.d. with prior $P_\Theta$.
Moreover, $\Theta$ is independent of $(T^K,Z^K)$.
Conditioned on $(\Theta_i,T_i)$, the verification output $V_i$ is generated by the verification channel in Assumption~\ref{asmp:verification-channel} with $\Theta$ replaced by $\Theta_i$.
The collection $(V_i)_{i\in\idx{K}}$ is conditionally independent across $i$ given $(\Theta,T^K)$.
\end{assumption}

\begin{theorem}[Score-based verification achieves a square-root gain]
\label{thm:score-achievability}
Assume Assumptions~\ref{asmp:screening}, \ref{asmp:verification-channel}, \ref{asmp:weak-screening}, and~\ref{asmp:decoupled-claims}.
Fix $\alpha\in(0,1)$ and set $B=\lfloor \alpha K\rfloor$.
We use the following policy.
We inspect $K$ i.i.d. records and observe $Z_1,\dots,Z_K$.
We compute the scores $\eta(Z_i)$.
We then verify the $B$ records with the largest scores.
Then, in the joint limit $\varepsilon\to 0$ and $K\to\infty$ with fixed $\alpha$,
\begin{multline}
\label{eq:score-achievability}
\Hh(\Theta)-D(K,B)
\;\ge\;
\min\Big\{\Hh(\Theta),\\
I_{\mathrm{ver}}\big(Bp + c_G(p,\alpha)\,\sqrt{J\,K\,B}\big)\Big\}
\; +\; o\!\big(\sqrt{J\,K\,B}\big)
\end{multline}
where $J=\I(T;Z)$ and
\begin{equation}
\label{eq:cG}
c_G(p,\alpha):=\sqrt{2\ln 2\,p(1-p)\,\alpha}\;m_G(\alpha).
\end{equation}
\end{theorem}

Theorem~\ref{thm:score-achievability} addresses the achievability gap.
It shows that a simple score-based policy attains the same square-root scaling as the converse.
The constant $m_G(\alpha)$ captures how much the policy can leverage the upper tail.

\begin{corollary}[Tight square-root scaling in a weak-screening regime]
\label{cor:tight-sqrt}
Under the assumptions of Theorem~\ref{thm:score-achievability}, let $D^\star(K,B)$ be the minimum Bayes log-loss over all policies that use budgets $(K,B)$.
Then, as $\varepsilon\to 0$ and $K\to\infty$ with fixed $\alpha\in(0,1)$ and $B=\lfloor \alpha K\rfloor$,
\begin{align}
\Hh(\Theta)-D^\star(K,B)
&\ge
I_{\mathrm{ver}}\big(Bp + c_G(p,\alpha)\sqrt{J K B}\big)\notag\\
&\qquad + o\!\big(\sqrt{J K B}\big),
\label{eq:inner-sqrt}
\\
\Hh(\Theta)-D^\star(K,B)
&\le
I_{\mathrm{ver}}\Big(Bp + \sqrt{\tfrac{\ln 2}{2}\,J K B}\Big).
\label{eq:outer-sqrt}
\end{align}
\end{corollary}

Corollary~\ref{cor:tight-sqrt} gives a nontrivial inner and outer bound pair.
Both bounds have the same $\sqrt{J K B}$ scaling.

\subsection{The haystack regime and tail leverage}
\label{subsec:haystack}

In many discovery systems, verification is much more expensive than inspection.
This corresponds to $B\ll K$, namely $\alpha=B/K\to 0$.
We call this the haystack regime.
In this regime, the effect of screening enters through the tail mean $m_G(\alpha)$.

Using \eqref{eq:cG} and the relation $\alpha=B/K$, we rewrite the leading gain term in \eqref{eq:score-achievability} as
\begin{equation}
\label{eq:haystack-rewrite}
c_G(p,\alpha)\,\sqrt{J K B}
\;=\;
\sqrt{2\ln 2\,p(1-p)}\;m_G(\alpha)\,\sqrt{J}\,B.
\end{equation}
Equation~\eqref{eq:haystack-rewrite} shows that expanding attention $K$ is valuable only if $m_G(B/K)$ grows.

We now record two canonical examples.
They show that the tail shape of the score determines the leverage of attention.

\begin{proposition}[Gaussian scores yield logarithmic tail leverage]
\label{prop:gaussian-tail}
Assume that $G\sim\mathcal N(0,1)$.
Then, as $\alpha\to 0$,
\begin{equation}
\label{eq:gaussian-tail}
m_G(\alpha)
=
\sqrt{2\log\!\frac{1}{\alpha}}\,\big(1+o(1)\big).
\end{equation}
\end{proposition}

Proposition~\ref{prop:gaussian-tail} implies that, for light-tailed scores, expanding $K$ yields diminishing returns.
The leverage grows only as $\sqrt{\log(K/B)}$.

\begin{proposition}[A Pareto right tail yields polynomial tail leverage]
\label{prop:pareto-tail}
Let $X$ have a Pareto distribution with exponent $\nu>3$, namely
$\Prob(X\ge x)=x^{-\nu}$ for $x\ge 1$.
Define the centered and standardized score
\[
G:=\frac{X-\E[X]}{\sqrt{\mathrm{Var}(X)}}.
\]
Then, as $\alpha\to 0$,
\begin{equation}
\label{eq:pareto-tail}
m_G(\alpha)
=
\frac{\nu}{\nu-1}\cdot\frac{1}{\sqrt{\mathrm{Var}(X)}}\,\alpha^{-1/\nu}\,\big(1+o(1)\big).
\end{equation}
\end{proposition}

Proposition~\ref{prop:pareto-tail} implies a very different scaling.
For a heavy-tailed score in the Fr\'echet domain, attention yields polynomial leverage.
This gives a precise criterion for when massive screening is most effective.

Fig.~\ref{fig:tail-leverage} visualizes these two regimes by plotting the asymptotic tail leverage $m_G(\alpha)$ (with $\alpha=B/K$) as a function of the oversampling ratio $K/B$.

\input{figures/tail_leverage_effect}

The key message is that pushing $K/B$ ever higher only pays off when the score distribution has exploitable extremes: for Gaussian-like tails the benefit grows only logarithmically, whereas for Pareto-like tails it grows as a power law.

\begin{remark}[A fully solved benchmark]
\label{rem:benchmark}
Appendix~\ref{app:tight-benchmark} gives a canonical model in which the risk-resource region can be characterized exactly.
The boundary is expressed through an upper-tail functional of the score distribution.
This benchmark complements the global-$\Theta$ analysis in this section.
\end{remark}

%% file: figures/tail_leverage_effect.tex
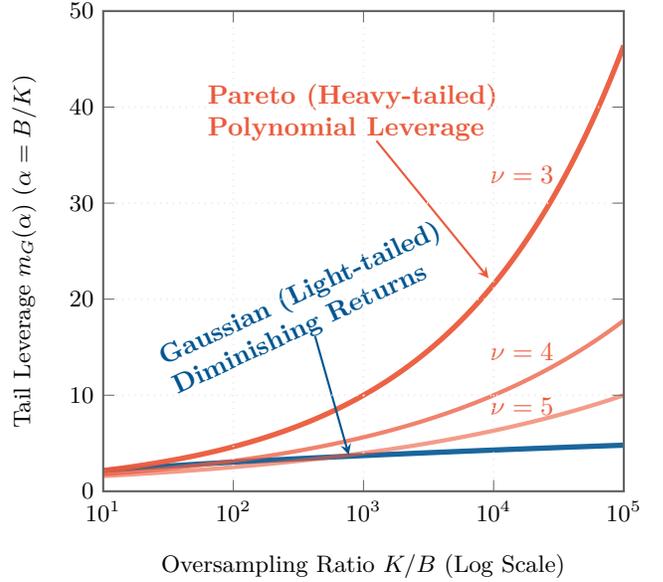
\begin{figure}[t]
\centering

\definecolor{techBlue}{RGB}{0, 85, 145}
\definecolor{techOrange}{RGB}{235, 90, 60}
\definecolor{gridGray}{RGB}{235, 235, 240}
\definecolor{axisGray}{RGB}{100, 100, 100}

\begin{tikzpicture}
    \begin{semilogxaxis}[
        width=0.96\linewidth,
        height=0.90\linewidth,
        font=\small, 
        axis on top,          
        axis line style={draw=axisGray, line width=0.8pt}, 
        grid=major,
        grid style={dotted, gridGray, line width=0.8pt}, 
        tick pos=both,        
        tick align=inside,    
        major tick style={draw=axisGray},
        xmin=10, xmax=100000,
        xlabel={Oversampling Ratio $K/B$ (Log Scale)},
        xmode=log,
        log basis x={10},
        xtick={10, 100, 1000, 10000, 100000},
        xticklabels={$10^1$, $10^2$, $10^3$, $10^4$, $10^5$},
        xlabel style={yshift=-0.2cm}, 
        ymin=0, ymax=50, 
        ylabel={Tail Leverage $m_G(\alpha)$ ($\alpha=B/K$)},
        ylabel style={yshift=0.2cm},
        legend style={draw=none, fill=none}
    ]

    \addplot[
        domain=10:100000, 
        samples=400, 
        color=techBlue, 
        line width=2pt, 
        opacity=0.9
    ]
    {sqrt(2 * ln(x))};
    

    \addplot[
        domain=10:100000,
        samples=400,
        color=techOrange,
        line width=1.6pt,
        opacity=0.60
    ]
    {x^(1/5)};

    \addplot[
        domain=10:100000,
        samples=400,
        color=techOrange,
        line width=1.6pt,
        opacity=0.75
    ]
    {x^(1/4)};

    \addplot[
        domain=10:100000,
        samples=400,
        color=techOrange,
        line width=2pt,
        opacity=0.95
    ]
    {x^(1/3)};

    
    \node[anchor=north west, color=techOrange, align=left, font=\bfseries, inner sep=0pt]
        (paretoLabel) at (rel axis cs: 0.20, 0.85) {Pareto (Heavy-tailed)\\Polynomial Leverage};

    \draw[-stealth, color=techOrange, thick, shorten >=3pt]
        (paretoLabel) -- (axis cs: 10000, {pow(10000,1/3)});

    \node[anchor=south west, color=techOrange, font=\bfseries]
        at (axis cs: 8000, {pow(30000,1/3)}) {$\nu=3$};
    \node[anchor=south west, color=techOrange, font=\bfseries]
        at (axis cs: 8000, {pow(26000,1/4)}) {$\nu=4$};
    \node[anchor=south west, color=techOrange, font=\bfseries]
        at (axis cs: 8000, {pow(17000,1/5)}) {$\nu=5$};

    \node[anchor=north west, color=techBlue, align=left, font=\bfseries, inner sep=0pt, rotate=25] 
        (gaussianLabel) at (axis cs: 25, 15) {Gaussian (Light-tailed)\\Diminishing Returns};
        
    \draw[-stealth, color=techBlue, thick, shorten >=3pt] 
        (gaussianLabel) -- (axis cs: 800, 3);

    \end{semilogxaxis}
\end{tikzpicture}
\caption{\textbf{Escaping the Gaussian trap: Tail leverage determines screening utility.}
The figure compares the asymptotic gain from massive screening ($K \gg B$) under light-tailed (Gaussian) versus heavy-tailed (Pareto with varying $\nu$) score distributions.
While Gaussian scores yield \emph{diminishing returns} (scaling logarithmically as $\sqrt{\ln K}$), heavy-tailed scores provide \emph{polynomial leverage} (scaling as $K^{1/\nu}$).
This illustrates the critical dichotomy in the haystack regime: expanding the screening budget $K$ is highly effective only when the score distribution admits exploitable extremes.}
\label{fig:tail-leverage}
\end{figure}

%% file: sections/conclusion_and_discussion.tex
\section{Discussion and Conclusion}
\label{sec:conclusion}

The post-generative information ecosystem has a new asymmetry. Producing and reshaping claims is cheap, while verifying them remains expensive. Modern pipelines can surface thousands of candidates through retrieval, heuristics, or lightweight model judgments, yet only a small subset can be investigated deeply with tools, experts, or reading. The binding constraint is often not access to text or the ability to phrase an answer, but the ability to allocate  attention to verification that turns plausible statements into warranted belief.

This paper formalizes that constraint through ACI. In each window, a decision-maker can screen many records at low cost and verify at most \(B\) of them at high cost after observing a screening statistic. Under Bayes log-loss, we measure performance by the maximum achievable reduction in posterior uncertainty per window, which we call \emph{epistemic throughput}. The main results characterize this throughput information-theoretically. The JaKoB scaling law isolates a simple interaction between screening quality \(J\), screening volume \(K\), and verification capacity \(B\): besides the linear baseline from verification, there is a leverage term of order \(\sqrt{JKB}\) (up to universal constants and the trivial ceiling \(\Hh(\Theta)\)). In sparse-verification regimes, leverage is governed by extreme values, which is why heavy-tailed score behavior can deliver polynomial amplification while light-tailed scores yield only logarithmic gains.

\subsection{Positioning: from reliability and semantics to truthfulness}
It is useful to compare ACI to established communication paradigms, because the distinction is about what is being optimized.

Traditional communication theory focuses on \emph{transmission reliability}. The core problem is to move symbols through a physical channel with limited bandwidth and noise. Bit rate and error probability are natural metrics because the message content is taken as given, and the challenge is preserving it through the medium.

Cognitive communication emphasizes \emph{resource adaptability}. The problem is to operate efficiently in a dynamic environment by sensing and exploiting time-varying resources, such as spectrum opportunities, under uncertainty and competition.

Semantic communication shifts attention to \emph{meaning consistency}. When many symbol strings can express the same intent, the objective becomes task success or semantic similarity rather than exact symbol recovery. The question is whether the receiver can act on what the sender meant.

The bottleneck motivating ACI is \emph{orthogonal} to all the three above. A claim can be transmitted perfectly in a resource-optimized way, interpreted exactly as intended, and still be false about the world. After large models became widely deployed, information streams contain an abundance of fluent, coherent statements whose truthfulness varies widely. When verification is scarce, the limiting factor is the ability to concentrate deep checks on the few items that can most improve the posterior. This is the regime in which epistemic throughput, rather than bit rate or semantic similarity, becomes the natural target.

\subsection{Toward a paradigm ``epistemic communication''}
These considerations motivate a possible next paradigm, which we call \emph{epistemic communication}. If this is a useful discipline, it should be defined by a distinct objective, a distinct metric, and a corresponding set of design questions.

\subsubsection{Objective and metric.}
Epistemic communication would focus on \emph{truthfulness} as a system-level goal: maximize the rate at which a receiver acquires verified, decision-relevant truths under attention constraints. Epistemic throughput offers a concrete metric because it measures posterior improvement under an operational loss (log-loss) and makes the verification budget explicit. Unlike a semantic metric, it distinguishes between statements that are merely consistent in meaning and statements that are supported by evidence strong enough to change a rational posterior.

\subsubsection{System boundary.}
Epistemic communication is inherently end-to-end. It spans how claims are produced and packaged, how weak signals are extracted cheaply, how verification actions are chosen, and how verification outputs are represented and shared as public artifacts. ACI is a minimal model that makes these components explicit. Screening corresponds to low-cost feature extraction and ranking; verification corresponds to expensive evidence acquisition; artifacts correspond to reusable proof objects, such as citations, provenance records, or standardized evaluation reports.

\subsubsection{Design principles suggested by ACI.}
The limit theorems in this paper suggest several principles that are likely to persist beyond the stylized assumptions used for analysis.

First, screening creates value only insofar as it increases \emph{selectivity}. In ACI, screening does not directly reveal \(\Theta\); it shapes which records get verified. In sparse-verification settings, average improvements in screening are often less important than the behavior of the top of the score distribution, because only the top few candidates are ever checked. This is exactly what the tail results capture: leverage is driven by rare, very high-scoring candidates.

Second, oversampling is a lever, but it has failure modes. The square-root term shows that increasing cheap screening volume can amplify scarce verification, but the amplification saturates quickly when scores are light-tailed. This yields a practical diagnostic: if doubling the screened pool does not materially improve the best few candidates, additional screening compute is unlikely to buy much epistemic progress. Conversely, engineering for heavy-tailed score behavior provides a concrete target: scoring pipelines should be designed to occasionally produce candidates that are far more verifiable and informative than typical ones.

Third, public artifacts are the medium of epistemic communication. Verification effort scales only if its outputs are cheap to reuse. This makes representation a first-class problem. Artifacts should be easy to consume, hard to counterfeit, and composable. Provenance metadata, standardized citation formats, cryptographic attestations, and structured argument summaries all fit naturally into this role. In ACI terms, better artifact design increases the downstream value of each verification action by enabling reuse across many decoders.

Finally, truthfulness is a systems property shaped by incentives and adversaries. In realistic ecosystems, the record stream is not passive. Producers may be rewarded for attention rather than accuracy, and adversaries can craft records whose screening features mimic truth. Epistemic communication therefore connects information theory to learning, mechanism design, and security: it asks how to align incentives and build auditability so that producing verifiable truth is the easiest way to succeed.

\subsection{Implications for retrieval-augmented and tool-using systems}
Retrieval-augmented generation and tool-use pipelines naturally fit the ACI template. Retrieval, query rewriting, and lightweight heuristics can surface many candidate snippets or hypotheses cheaply, while deep follow-up is limited by expensive tool calls, constrained context windows, human review, or time. The theory suggests evaluating screening stages by their contribution to \emph{verification yield}, not only by average ranking quality. In particular, the heavy-tail requirement points to an engineering goal: screening should sometimes produce ``clean hits'' that are easy to verify and highly informative once verified, rather than merely shifting average ranks.

The model also highlights the value of reusing verification work. Caching verified sources, publishing structured citations, and storing intermediate tool traces can turn a one-off verification action into a reusable artifact. This does not increase the raw verification budget \(B\), but it increases the epistemic return per verification by reducing repeated effort across users and across time.

\subsection{Limitations and open directions}
Our analysis is intentionally stylized, and its limitations outline a research agenda. First, the present model is window-based and largely i.i.d.; real systems are sequential and reflexive. Screening and verification decisions affect future data collection, user behavior, and even the content stream itself. Extending epistemic throughput to sequential settings would connect directly to active learning, bandits, and adaptive experimentation, where verification actions are chosen to maximize long-run posterior improvement.

Second, we focus on log-loss because it links optimal risk to conditional entropy and yields clean converse bounds. Other objectives, such as decision regret, calibration under abstention, or constrained false discovery, may induce different throughput notions and different limits. Third, verification costs are heterogeneous in practice and evidence sources are correlated. A richer model would allow multi-level verification ladders with varying costs and dependencies, which would better match tool chains that mix cheap checks with a few expensive audits.

Finally, the heavy-tail phenomenon raises a practical question: what architectural choices create heavy-tailed score distributions that are both useful and robust? Answering this would bridge the theory to concrete designs for ranking functions, tool orchestration, and evidence aggregation in the presence of strategic manipulation.

\subsection{Closing perspective}
When content is abundant and attention is scarce, the binding constraint is the rate at which systems can convert weak, cheap signals into strong, verified updates. ACI provides a language for this regime and a scaling law that quantifies its best-case behavior. More broadly, it motivates epistemic communication as a principled way to study truthfulness at scale, with epistemic throughput serving as a system-level metric for progress.

%% file: sections/appendices.tex
\begin{appendix}
    
\input{sections/tight_region}

\section{Proofs for Section~\ref{sec:theory}}
\label{sec:proofs}

We provide proofs of the main results.
All logarithms are base $2$.

\subsection{Proof of Lemma~\ref{lem:enrichment}}
\begin{proof}
Let $S$ be any randomized function of $Z$ and let $\alpha=\Prob(S=1)$.
By data processing, $\I(T;S)\le \I(T;Z)=J$.
We expand $\I(T;S)$ using conditional relative entropy:
\[
\I(T;S)
=
\alpha\,\KL\big(P_{T\mid S=1}\,\|\,P_T\big)
+
(1-\alpha)\,\KL\big(P_{T\mid S=0}\,\|\,P_T\big).
\]
Both KL terms are nonnegative, hence
\[
\KL\big(P_{T\mid S=1}\,\|\,P_T\big) \le \frac{1}{\alpha}\,\I(T;S) \le \frac{J}{\alpha}.
\]

We now relate KL divergence to total variation.
Pinsker's inequality for base-$2$ KL divergence states that
\[
\mathrm{TV}\big(P_{T\mid S=1},P_T\big)
\le
\sqrt{\frac{\ln 2}{2}\,\KL\big(P_{T\mid S=1}\,\|\,P_T\big)}.
\]
Since $T$ is binary, total variation equals the absolute difference of the probability of $T=1$:
\[
\mathrm{TV}\big(P_{T\mid S=1},P_T\big)
=
\abs{\Prob(T=1\mid S=1)-\Prob(T=1)}.
\]
Combining the last three displays yields
\[
\Prob(T=1\mid S=1)
\le
p + \sqrt{\frac{\ln 2}{2\alpha}\,J}.
\]
This is \eqref{eq:enrichment}.
\end{proof}

\subsection{Proof of Theorem~\ref{thm:ver-tradeoff}}
\begin{proof}
We consider $K$ inspected records indexed by $i\in\idx{K}$.
For record $i$, let $(T_i,Z_i,V_i)$ be distributed as in Assumptions~\ref{asmp:screening} and \ref{asmp:verification-channel}.
Let $S_i\in\{0,1\}$ indicate whether record $i$ is verified.
The policy may select the verified set using the entire vector $\bm{Z}=(Z_1,\dots,Z_K)$ and private randomness independent of $(T^K,Z^K)$.
The budget constraint $\sum_{i=1}^K S_i\le B$ holds almost surely.

For each record $i$, define the verification transcript
\[
Y_i:=\begin{cases}
\bot, & S_i=0,\\
(T_i,V_i), & S_i=1,
\end{cases}
\]
where $\bot$ is a fixed symbol, and let $\bm{Y}:=(Y_1,\dots,Y_K)$.
Under Bayes-optimal prediction, the expected log-loss equals
$D(K,B)=\Hh(\Theta\mid \bm{Z},\bm{Y})$.
Therefore,
\[
\Hh(\Theta)-D(K,B)=\I(\Theta;\bm{Z},\bm{Y}).
\]
Assumption~\ref{asmp:screening} implies that $(T^K,Z^K)$ is independent of $\Theta$, hence $\I(\Theta;\bm{Z})=0$ and
\[
\Hh(\Theta)-D(K,B)=\I(\Theta;\bm{Y}\mid \bm{Z}).
\]

We upper bound $\I(\Theta;\bm{Y}\mid \bm{Z})$ by a chain-rule argument that keeps the global latent state.
By the chain rule,
\[
\I(\Theta;\bm{Y}\mid \bm{Z})
=
\sum_{i=1}^K \I\big(\Theta;Y_i\mid \bm{Z},Y^{i-1}\big).
\]
If $S_i=0$ then $Y_i$ is constant and the term is zero.
If $S_i=1$ then $Y_i=(T_i,V_i)$.
Since $T_i$ is independent of $\Theta$, we have
\[
\I\big(\Theta;Y_i\mid \bm{Z},Y^{i-1},S_i=1\big)
=
\I\big(\Theta;V_i\mid T_i,\bm{Z},Y^{i-1},S_i=1\big).
\]

Let $W_i:=(\bm{Z},Y^{i-1},S_i=1)$.
By the record-generation model in Section~\ref{sec:model} together with Assumption~\ref{asmp:verification-channel}, we have the conditional independence
$V_i\perp\!\!\!\perp W_i\mid (\Theta,T_i)$.
Using the identity
$\I(\Theta;V_i\mid T_i,W_i)=\Hh(V_i\mid T_i,W_i)-\Hh(V_i\mid \Theta,T_i)$,
we obtain
\[
\I(\Theta;V_i\mid T_i,W_i)
\le
\Hh(V_i\mid T_i)-\Hh(V_i\mid \Theta,T_i)
=
\I(\Theta;V_i\mid T_i).
\]
Therefore,
\[
\I(\Theta;\bm{Y}\mid \bm{Z})
\le
\sum_{i=1}^K \E\bracks{S_i\,\I(\Theta;V_i\mid T_i)}.
\]
If $T_i=0$ then $V_i$ is independent of $\Theta$ and the mutual information term is $0$.
If $T_i=1$ the term equals $I_{\mathrm{ver}}$ by \eqref{eq:iver}.
Thus
\[
\I(\Theta;\bm{Y}\mid \bm{Z})
\le
I_{\mathrm{ver}}\sum_{i=1}^K \E\bracks{S_i\,\1\{T_i=1\}}.
\]

We now bound the expected number of informative verified records.
Let $I$ be a random index uniform on $\idx{K}$ and independent of all other variables.
Define $(T,Z):=(T_I,Z_I)$ and $S:=S_I$.
Then
\[
\alpha:=\Prob(S=1)=\E[S]=\frac{1}{K}\sum_{i=1}^K \E[S_i]\le \frac{B}{K}.
\]
By the definition of the random index $I$, we have
\[
\E\bracks{S\,\1\{T=1\}}=\frac{1}{K}\sum_{i=1}^K \E\bracks{S_i\,\1\{T_i=1\}}.
\]
By the definition of conditional probability,
\begin{multline*}
\E\bracks{S\,\1\{T=1\}}=\Prob(S=1)\Prob(T=1\mid S=1) \\ 
=\alpha\Prob(T=1\mid S=1).
\end{multline*}
Combining yields
\[
\sum_{i=1}^K \E\bracks{S_i\,\1\{T_i=1\}}=K\alpha\,\Prob(T=1\mid S=1).
\]

We bound $\Prob(T=1\mid S=1)$ in terms of $J=\I(T;Z)$.
Since $\bm{Z}_{-I}$ is independent of $(T,Z)$, the Markov chain
\[
T\; -\; Z\; -\; (S,\bm{Z}_{-I})
\]
holds.
By data processing, $\I(T;S)\le \I(T;Z)=J$.
Lemma~\ref{lem:enrichment} applies and yields
\[
\Prob(T=1\mid S=1)
\le
p + \sqrt{\frac{\ln 2}{2\alpha}\,J}.
\]
Combining gives
\[
\I(\Theta;\bm{Y}\mid \bm{Z})
\le
I_{\mathrm{ver}}\,K\alpha\Bigg(p + \sqrt{\frac{\ln 2}{2\alpha}\,J}\Bigg).
\]
Since $K\alpha\le B$, we have $K\alpha p\le Bp$.
For the square-root term we use $K\alpha\sqrt{1/\alpha}=K\sqrt{\alpha}\le \sqrt{KB}$.
This yields
\begin{multline*}
\Hh(\Theta)-D(K,B)
\le
I_{\mathrm{ver}}\Bigg(Bp + \sqrt{\frac{\ln 2}{2}\,J\,BK}\Bigg) \\
=
B\cdot I_{\mathrm{ver}}\Bigg(
p + \sqrt{\frac{\ln 2}{2}}\,\sqrt{\frac{JK}{B}}
\Bigg).
\end{multline*}
This is \eqref{eq:ver-tradeoff}.
\end{proof}

\subsection{Proof of Corollary~\ref{cor:ver-budget}}
\begin{proof}
Assume that a policy satisfies $\Hh(\Theta)-D(K,B)\ge \Delta$.
If $I_{\mathrm{ver}}=0$, then Theorem~\ref{thm:ver-tradeoff} implies $\Hh(\Theta)-D(K,B)=0$ for every policy.
Since $\Delta>0$, the premise cannot hold and the corollary is vacuous.
We therefore assume $I_{\mathrm{ver}}>0$ in the remainder of the proof.
By Theorem~\ref{thm:ver-tradeoff},
\[
\Delta
\le
I_{\mathrm{ver}}\Bigg(Bp + \sqrt{\frac{\ln 2}{2}\,J\,K\,B}\Bigg).
\]
We divide both sides by $I_{\mathrm{ver}}$ and define
\[
g:=\frac{\Delta}{I_{\mathrm{ver}}}
\qquad\text{and}\qquad
a:=\sqrt{\frac{\ln 2}{2}\,J\,K}.
\]
The inequality becomes $pB+a\sqrt{B}\ge g$.
Let $x:=\sqrt{B}$.
We rewrite the constraint as a quadratic inequality
\[
p x^2 + a x - g \ge 0.
\]
Since $p>0$, the polynomial is convex in $x$ and has a unique nonnegative root.
Therefore
\[
x \ge \frac{-a+\sqrt{a^2+4pg}}{2p}.
\]
Squaring both sides yields
\[
B \ge \frac{1}{4p^2}\big(\sqrt{a^2+4pg}-a\big)^2.
\]
Substituting the definitions of $a$ and $g$ gives \eqref{eq:ver-budget}.
\end{proof}

\begin{lemma}[Averages of top-score samples]
\label{lem:top-alpha}
Let $G_1,G_2,\dots$ be i.i.d. copies of a random variable $G$ such that $\E[G^2]<\infty$.
Assume that $G$ has a continuous distribution.
Fix $\alpha\in(0,1)$ and set $B=\lfloor \alpha K\rfloor$.
Let $G_{(1)}\ge\cdots\ge G_{(K)}$ denote the order statistics of $(G_1,\dots,G_K)$.
Then, as $K\to\infty$,
\[
\frac{1}{B}\sum_{i=1}^B G_{(i)} \to m_G(\alpha)
\qquad\text{in probability,}
\]
where $m_G(\alpha)$ is defined in \eqref{eq:tail-mean}.
Moreover, the convergence also holds in $\mathsf{L}^1$.
\end{lemma}

\subsection{Proof of Lemma~\ref{lem:top-alpha}}
\begin{proof}
We first prove convergence in probability.
Fix $\alpha\in(0,1)$ and set $B=\lfloor \alpha K\rfloor$.
Let $q_\alpha$ be a $(1-\alpha)$-quantile of $G$ as in \eqref{eq:tail-mean}.

Fix $\delta>0$.
Since $q_\alpha$ is a quantile, we have $\Prob(G\ge q_\alpha+\delta)<\alpha$ and $\Prob(G>q_\alpha-\delta)>\alpha$.
Define the counting variables
\[
N_+ := \sum_{i=1}^K \1\{G_i\ge q_\alpha+\delta\},
\qquad
N_- := \sum_{i=1}^K \1\{G_i> q_\alpha-\delta\}.
\]
By the strong law of large numbers, $N_+/K\to \Prob(G\ge q_\alpha+\delta)$ and $N_-/K\to \Prob(G>q_\alpha-\delta)$ almost surely.
Therefore, the event $\{N_+<B<N_-\}$ holds with probability tending to $1$.

On this event, every sample with $G_i\ge q_\alpha+\delta$ must appear among the top $B$ values, and every top-$B$ sample must satisfy $G_i>q_\alpha-\delta$.
This yields the sandwich bound
\begin{align*}
\frac{1}{B}\sum_{i=1}^K G_i\,\1\{G_i\ge q_\alpha+\delta\}
&\le
\frac{1}{B}\sum_{i=1}^B G_{(i)} \\
&\le
\frac{1}{B}\sum_{i=1}^K G_i\,\1\{G_i> q_\alpha-\delta\},
\end{align*}
with probability tending to $1$.

Since $\E[G^2]<\infty$, the random variables $G\,\1\{G\ge q_\alpha+\delta\}$ and $G\,\1\{G> q_\alpha-\delta\}$ are integrable.
The strong law gives
\[
\frac{1}{K}\sum_{i=1}^K G_i\,\1\{G_i\ge q_\alpha+\delta\}
\to
\E\bracks{G\,\1\{G\ge q_\alpha+\delta\}}
\qquad\text{a.s.}
\]
and the same for the upper bracket.
Since $K/B\to 1/\alpha$, we obtain almost sure limits for the two normalized sums by multiplying by $K/B$.

We now let $\delta\downarrow 0$.
As $\delta\downarrow 0$, we have the pointwise convergences
\begin{align*}
\1\{G\ge q_\alpha+\delta\} & \to \1\{G>q_\alpha\}, \\
\1\{G>q_\alpha-\delta\}    & \to \1\{G\ge q_\alpha\}.
\end{align*}
Since $\E[\abs{G}]<\infty$ and the integrands are dominated by $\abs{G}$, dominated convergence yields
\begin{align*}
\E\bracks{G\,\1\{G\ge q_\alpha+\delta\}} & \to \E\bracks{G\,\1\{G>q_\alpha\}}, \\
\E\bracks{G\,\1\{G>q_\alpha-\delta\}} & \to \E\bracks{G\,\1\{G\ge q_\alpha\}}.
\end{align*}
Because $G$ has a continuous distribution, $\Prob(G=q_\alpha)=0$, hence
$\E\bracks{G\,\1\{G>q_\alpha\}}=\E\bracks{G\,\1\{G\ge q_\alpha\}}$.
Therefore, both limits coincide and equal $\E\bracks{G\,\1\{G\ge q_\alpha\}}$.
This implies that the sandwich bounds converge to $m_G(\alpha)$.
Hence $\frac{1}{B}\sum_{i=1}^B G_{(i)}\to m_G(\alpha)$ in probability.

We next show convergence in $\mathsf{L}^1$.
By Cauchy--Schwarz,
\[
\abs{\frac{1}{B}\sum_{i=1}^B G_{(i)}}
\le
\sqrt{\frac{1}{B}\sum_{i=1}^B G_{(i)}^2}.
\]
Since the top-$B$ samples are a subset of all $K$ samples, we have
$\sum_{i=1}^B G_{(i)}^2 \le \sum_{i=1}^K G_i^2$.
Taking expectations yields
\[
\E\bracks{\frac{1}{B}\sum_{i=1}^B G_{(i)}^2}
\le
\frac{K}{B}\,\E[G^2]
=
\frac{1}{\alpha}\,\E[G^2].
\]
Hence $\{\frac{1}{B}\sum_{i=1}^B G_{(i)}\}_{K\ge 1}$ is uniformly integrable.
Convergence in probability together with uniform integrability implies $\mathsf{L}^1$ convergence.
\end{proof}

\subsection{Proof of Theorem~\ref{thm:score-achievability}}
\begin{proof}
We consider $K$ i.i.d. records indexed by $i\in\idx{K}$.
For record $i$, let $(T_i,Z_i)$ satisfy Assumption~\ref{asmp:screening}.
If record $i$ is verified, the decoder observes $(T_i,V_i)$ as in Assumption~\ref{asmp:verification-channel}.
We write $\eta_i:=\Prob(T_i=1\mid Z_i)$.
By Assumption~\ref{asmp:weak-screening},
\[
\ln\frac{\eta_i}{1-\eta_i}
=
\ln\frac{p}{1-p}+\varepsilon G_i,
\]
where $G_1,\dots,G_K$ are i.i.d. copies of $G$.
The map $g\mapsto \ln\frac{\eta}{1-\eta}$ is strictly increasing in $\eta\in(0,1)$.
Therefore $\eta_i$ is a strictly increasing function of $G_i$.
The policy verifies the $B=\lfloor \alpha K\rfloor$ largest scores $\eta_i$, which is equivalent to verifying the $B$ largest values among $G_1,\dots,G_K$.
We denote the verified index set by $\cB$.

Under Bayes-optimal prediction, observing the full transcript is equivalent to observing
$(Z^K,\cB,T_{\cB},V_{\cB})$, where $Z^K:=(Z_1,\dots,Z_K)$ and $\cB\subseteq\idx{K}$ is the verified index set.
Therefore,
\[
D(K,B)=\Hh\big(\Theta\mid Z^K,\cB,\,T_{\cB},\,V_{\cB}\big)
\]
and
\[
\Hh(\Theta)-D(K,B)=\I\big(\Theta;Z^K,\cB,\,T_{\cB},\,V_{\cB}\big).
\]

Assumption~\ref{asmp:screening} implies that $(T^K,Z^K)$ is independent of $\Theta$.
For any fixed $\varepsilon>0$, since $G$ has a continuous distribution, the scores are distinct with probability one.
Therefore the top-$B$ verification set $\cB$ is a deterministic function of $Z^K$.
Hence $(Z^K,\cB,T_{\cB})$ is independent of $\Theta$, and
\[
\Hh(\Theta)-D(K,B)=\I\big(\Theta;V_{\cB}\mid Z^K,\cB,T_{\cB}\big).
\]

Under Assumptions~\ref{asmp:verification-channel} and \ref{asmp:decoupled-claims}, conditioned on $(\Theta,T^K)$,
the verification outputs are generated without using $Z^K$.
Since $\Theta$ is also independent of $Z^K$, this implies that, conditioned on $(\cB,T_{\cB})$,
the pair $(\Theta,V_{\cB})$ is independent of $Z^K$.
Therefore
\[
\I\big(\Theta;V_{\cB}\mid Z^K,\cB,T_{\cB}\big)
=
\I\big(\Theta;V_{\cB}\mid \cB,T_{\cB}\big).
\]

We now expand $\I(\Theta;V_{\cB}\mid \cB,T_{\cB})$.
By Assumption~\ref{asmp:decoupled-claims}, conditioned on $(\cB,T_{\cB})$, the variables $V_i$ are independent across $i\in\cB$.
Moreover, $V_i$ depends on $\Theta$ only through $\Theta_i$.
Therefore
\begin{align}
\I\big(\Theta;V_{\cB}\mid \cB,T_{\cB}\big)
&=\Hh(V_{\cB}\mid \cB,T_{\cB})-\Hh(V_{\cB}\mid \Theta,\cB,T_{\cB})\notag\\
&=\sum_{i\in\cB}\Hh(V_i\mid T_i)-\sum_{i\in\cB}\Hh(V_i\mid \Theta_i,T_i)\notag\\
&=\sum_{i\in\cB}\I\big(\Theta_i;V_i\mid T_i\big).
\end{align}
If $T_i=0$ the term is $0$.
If $T_i=1$ the term equals $I_{\mathrm{ver}}$ by \eqref{eq:iver}.
Thus
\begin{equation}
\label{eq:gain-Iver-app}
\Hh(\Theta)-D(K,B)=I_{\mathrm{ver}}\cdot \E\bracks{\sum_{i\in\cB}\1\{T_i=1\}}.
\end{equation}
We also have the trivial upper bound $\Hh(\Theta)-D(K,B)\le \Hh(\Theta)$.
Combining yields
\begin{equation}
\label{eq:gain-min-app}
\Hh(\Theta)-D(K,B)
\ge
\min\Big\{\Hh(\Theta),\;I_{\mathrm{ver}}\cdot \E\bracks{\sum_{i\in\cB}\1\{T_i=1\}}\Big\}.
\end{equation}

We now lower bound $\E\big[\sum_{i\in\cB}\1\{T_i=1\}\big]$.
By the tower property,
\[
\E[\1\{i\in\cB\}T_i]=\E\big[\E[T_i\mid Z_i]\,\1\{i\in\cB\}\big]=\E[\eta_i\,\1\{i\in\cB\}].
\]
Summing over $i$ gives
\[
\E\Big[\sum_{i\in\cB}\1\{T_i=1\}\Big]=\sum_{i=1}^K \E[\eta_i\,\1\{i\in\cB\}].
\]
We next relate $\eta_i$ to $G_i$.
Let $\sigma(x):=(1+e^{-x})^{-1}$ be the logistic function and let $s:=\ln\frac{p}{1-p}$.
By Assumption~\ref{asmp:weak-screening}, we have $\eta_i=\sigma(s+\varepsilon G_i)$.
Since $\sigma(s)=p$ and $\sigma'(s)=p(1-p)$, Taylor's theorem gives
\[
\eta_i
=
p+\varepsilon p(1-p)G_i+R_i,
\]
where the remainder satisfies $\abs{R_i}\le c_0\varepsilon^2 G_i^2$ for a constant $c_0$ that depends only on $p$.

Using the tower property and $\E[T_i\mid Z_i]=\eta_i$, we have
\[
\E\Big[\sum_{i\in\cB}\1\{T_i=1\}\Big]=\sum_{i=1}^K \E[\eta_i\,\1\{i\in\cB\}].
\]
Substituting the expansion of $\eta_i$ yields
\begin{multline*}
\E\Big[\sum_{i\in\cB}\1\{T_i=1\}\Big]
=
Bp  \\ 
+ \varepsilon p(1-p)\sum_{i=1}^K \E[G_i\,\1\{i\in\cB\}]  + \sum_{i=1}^K \E[R_i\,\1\{i\in\cB\}].
\end{multline*}
Let $G_{(1)}\ge\cdots\ge G_{(K)}$ denote the order statistics.
Since $\cB$ selects the top $B$ scores, $\sum_{i=1}^K G_i\,\1\{i\in\cB\}=\sum_{i=1}^B G_{(i)}$.
Lemma~\ref{lem:top-alpha} implies
\[
\frac{1}{B}\sum_{i=1}^B G_{(i)}\to m_G(\alpha)
\qquad\text{in probability as }K\to\infty.
\]
Thus
\[
\sum_{i=1}^K \E[G_i\,\1\{i\in\cB\}] = B\,m_G(\alpha)+o(B),
\qquad\text{as }K\to\infty.
\]
We now bound the remainder term.
Since $\sum_{i\in\cB} G_i^2\le \sum_{i=1}^K G_i^2$, we have
\[
\sum_{i=1}^K \E\bracks{G_i^2\,\1\{i\in\cB\}}\le K\E[G^2]=K.
\]
Using $\abs{R_i}\le c_0\varepsilon^2 G_i^2$ yields
\[
\abs{\sum_{i=1}^K \E[R_i\,\1\{i\in\cB\}]}
\le
c_0\varepsilon^2\sum_{i=1}^K \E\bracks{G_i^2\,\1\{i\in\cB\}}
\le
c_0\varepsilon^2 K.
\]
Since $B=\lfloor \alpha K\rfloor$ with fixed $\alpha\in(0,1)$ and $\varepsilon\to 0$, the bound $c_0\varepsilon^2 K$ is $o(\varepsilon B)$.
We obtain
\begin{equation}
\label{eq:exp-inf-app}
\E\Big[\sum_{i\in\cB}\1\{T_i=1\}\Big]
=
Bp + \varepsilon p(1-p)\,B\,m_G(\alpha)+o(\varepsilon B).
\end{equation}

We now relate $\varepsilon$ to $J=\I(T;Z)$.
For binary $T$,
\[
J=\I(T;Z)=\E\bracks{\KL\big(\mathrm{Bern}(\eta(Z))\,\|\,\mathrm{Bern}(p)\big)}.
\]
Let $\delta:=\eta(Z)-p$.
Under Assumption~\ref{asmp:weak-screening}, Taylor's theorem for the logistic function gives
$\delta=\varepsilon p(1-p)G+O(\varepsilon^2 G^2)$.
The binary relative entropy admits the expansion
\begin{equation}
\label{eq:kl-taylor-app}
\KL\big(\mathrm{Bern}(p+x)\,\|\,\mathrm{Bern}(p)\big)
=
\frac{x^2}{2\ln 2\,p(1-p)}+o(x^2).
\end{equation}
Since $\E[\abs{G}^3]<\infty$, we can take expectations in \eqref{eq:kl-taylor-app} with $x=\delta$.
This yields
\[
\begin{aligned}
J
&=\frac{\E[\delta^2]}{2\ln 2\,p(1-p)}+o(\varepsilon^2)\\
&=\frac{p(1-p)\,\varepsilon^2}{2\ln 2}+o(\varepsilon^2),
\qquad \varepsilon\to 0.
\end{aligned}
\]
Therefore
\[
\varepsilon p(1-p)=\sqrt{2\ln 2\,p(1-p)\,J}+o(\sqrt{J}).
\]

We substitute into \eqref{eq:exp-inf-app}.
Since $B=\lfloor \alpha K\rfloor$, we have $B\sqrt{J}=\sqrt{\alpha}\,\sqrt{J K B}\,\big(1+o(1)\big)$.
We obtain
\begin{align*}
\E\Big[\sum_{i\in\cB}\1\{T_i=1\}\Big]
&=
Bp\\
&\quad+\sqrt{2\ln 2\,p(1-p)\,\alpha}\;m_G(\alpha)\,\sqrt{J K B}\\
&\quad+o\big(\sqrt{J K B}\big).
\end{align*}
Combining with \eqref{eq:gain-min-app} yields \eqref{eq:score-achievability}.
\end{proof}

\subsection{Proof of Corollary~\ref{cor:tight-sqrt}}
\begin{proof}
The upper bound is Theorem~\ref{thm:ver-tradeoff}.
For the lower bound, $D^\star(K,B)$ is the minimum Bayes log-loss over all feasible policies.
It is therefore no smaller than the log-loss achieved by the score-based policy in Theorem~\ref{thm:score-achievability}.
\end{proof}

\subsection{Proof of Proposition~\ref{prop:gaussian-tail}}
\begin{proof}
For a continuous distribution, \eqref{eq:tail-mean} reduces to a conditional mean.
Let $q_\alpha$ be the $(1-\alpha)$-quantile of $G$.
For $G\sim\mathcal N(0,1)$,
\[
m_G(\alpha)=\E[G\mid G\ge q_\alpha].
\]
A standard identity for the truncated normal gives
\begin{equation}
\label{eq:trunc-normal-app}
\E[G\mid G\ge q]=\frac{\varphi(q)}{1-\Phi(q)},
\end{equation}
where $\varphi(q)=\frac{1}{\sqrt{2\pi}}e^{-q^2/2}$ and $\Phi$ is the standard normal cdf.

We bound $1-\Phi(q)$ in terms of $\varphi(q)$ for $q>0$.
Using integration by parts and the identity $\varphi'(u)=-u\,\varphi(u)$ yields the upper bound
\begin{equation}
\label{eq:mills-upper-app}
1-\Phi(q)\le \frac{\varphi(q)}{q}.
\end{equation}
A second integration by parts gives the lower bound
\begin{equation}
\label{eq:mills-lower-app}
1-\Phi(q)\ge \frac{\varphi(q)}{q+1/q}.
\end{equation}
Substituting \eqref{eq:mills-upper-app} and \eqref{eq:mills-lower-app} into \eqref{eq:trunc-normal-app} yields
\[
q_\alpha \le m_G(\alpha)\le q_\alpha+\frac{1}{q_\alpha}.
\]
Therefore, it is enough to identify the leading growth of $q_\alpha$.

Since $\alpha=1-\Phi(q_\alpha)$, the bound \eqref{eq:mills-upper-app} gives
\[
\alpha\le \frac{1}{\sqrt{2\pi}}\,\frac{e^{-q_\alpha^2/2}}{q_\alpha}.
\]
Taking logarithms yields
\[
\frac{q_\alpha^2}{2}\le \log\frac{1}{\alpha} - \log q_\alpha + O(1).
\]
This implies $q_\alpha^2\le 2\log(1/\alpha)+O(\log\log(1/\alpha))$.
The bound \eqref{eq:mills-lower-app} gives a matching lower bound and yields
\[
q_\alpha=\sqrt{2\log\frac{1}{\alpha}}\,(1+o(1)).
\]
Since $m_G(\alpha)$ differs from $q_\alpha$ by at most $1/q_\alpha$, we obtain \eqref{eq:gaussian-tail}.
\end{proof}

\subsection{Proof of Proposition~\ref{prop:pareto-tail}}
\begin{proof}
Let $X$ have the Pareto tail $\Prob(X\ge x)=x^{-\nu}$ for $x\ge 1$, with $\nu>3$.
The mean and variance are
\[
\mu:=\E[X]=\frac{\nu}{\nu-1},
\qquad
\sigma^2:=\mathrm{Var}(X)=\frac{\nu}{(\nu-1)^2(\nu-2)}.
\]
Define $G=(X-\mu)/\sigma$.
For $\alpha\in(0,1)$, let $q_\alpha$ be the $(1-\alpha)$-quantile of $G$.
Since $G$ is an increasing function of $X$, we have
\[
\Prob(G\ge q)=\Prob(X\ge \sigma q+\mu).
\]
The $(1-\alpha)$-quantile of $X$ is $x_\alpha=\alpha^{-1/\nu}$.
Therefore $q_\alpha=(x_\alpha-\mu)/\sigma$.

For a continuous distribution, $m_G(\alpha)=\E[G\mid G\ge q_\alpha]$.
We compute
\begin{align*}
\E\big[G\,\1\{G\ge q_\alpha\}\big]
&=
\frac{1}{\sigma}\,\E\big[(X-\mu)\,\1\{X\ge x_\alpha\}\big] \\
&=
\frac{1}{\sigma}\,\Big(\E[X\,\1\{X\ge x_\alpha\}] - \mu\,\Prob(X\ge x_\alpha)\Big).
\end{align*}
For the Pareto density $f(x)=\nu x^{-\nu-1}$ on $x\ge 1$, we have
\begin{multline*}
\E[X\,\1\{X\ge x_\alpha\}]
=
\int_{x_\alpha}^{\infty} x f(x)\,dx \\
=
\nu\int_{x_\alpha}^{\infty} x^{-\nu}\,dx
=
\frac{\nu}{\nu-1}\,x_\alpha^{1-\nu}.
\end{multline*}
We also have $\Prob(X\ge x_\alpha)=x_\alpha^{-\nu}$.
Dividing by $\alpha=\Prob(X\ge x_\alpha)$ yields
\begin{multline*}
m_G(\alpha)
=
\E[G\mid G\ge q_\alpha] \\
=
\frac{1}{\sigma}\Big(\frac{\nu}{\nu-1}\,x_\alpha-\mu\Big)
=
\frac{1}{\sigma}\Big(\frac{\nu}{\nu-1}\,\alpha^{-1/\nu}-\mu\Big).
\end{multline*}
As $\alpha\to 0$, the additive constant $\mu/\sigma$ is negligible compared to $\alpha^{-1/\nu}$.
This gives \eqref{eq:pareto-tail}.
\end{proof}

\section{Proofs for Appendix~\ref{app:tight-benchmark}}
\label{sec:proofs-tight}

\subsection{Proof of Theorem~\ref{thm:optimal-hitrate}}
\label{app:proof-hitrate}
\begin{proof}
Fix any verification policy and its verification set $\cB\subseteq\idx{K}$.
We condition on the inspection outcomes $Z^K:=(Z_1,\dots,Z_K)$.
Under the benchmark model, types are conditionally independent given $Z^K$, and $\E[T_i\mid Z_i]=\eta(Z_i)=:\eta_i$.
Therefore,
\[
\E\Big[\sum_{i\in\cB} T_i \mid Z^K\Big]
=
\sum_{i\in\cB} \E[T_i\mid Z^K]
=
\sum_{i\in\cB} \eta_i.
\]
For fixed scores $\eta_1,\dots,\eta_K$ and a fixed budget $\abs{\cB}=B$, the sum is maximized by choosing the indices of the $B$ largest scores.
Let $\eta_{(1)}\ge\cdots\ge\eta_{(K)}$ denote the order statistics.
Then
\[
\sum_{i\in\cB} \eta_i \le \sum_{\ell=1}^{B} \eta_{(\ell)}.
\]
Taking expectations over $Z^K$ yields \eqref{eq:hit-rate-opt}.
The upper bound is achieved by verifying the $B$ largest scores.
\end{proof}

\subsection{Proof of Theorem~\ref{thm:tight-region}}
\label{app:proof-region}
\begin{proof}
Consider a verified record $i\in\cB$.
If $T_i=0$, then by Assumption~\ref{asmp:verification-sidechannel} the verification output is the constant symbol $V_i=v_0$.
The Bayes optimal prediction is the prior $P_\Theta$, and the expected log-loss equals $\Hh(\Theta)$.
If $T_i=1$, then $V_i$ is generated by $P_{V\mid\Theta}$.
The Bayes optimal prediction is the posterior $P(\Theta_i\mid V_i)$, and the expected log-loss equals $\Hh(\Theta\mid V)$.
By the definition of mutual information, $\Hh(\Theta\mid V)=\Hh(\Theta)-I_{\mathrm{ver}}^{\mathrm{bm}}$.
Therefore,
\[
\E\bracks{-\log q_i(\Theta_i)}
=
\Hh(\Theta)-I_{\mathrm{ver}}^{\mathrm{bm}}\,\E[T_i].
\]
Averaging over the $B$ verified records yields
\[
\bar{D}(K,B)
=
\Hh(\Theta)-\frac{I_{\mathrm{ver}}^{\mathrm{bm}}}{B}\,\E\Big[\sum_{i\in\cB} T_i\Big].
\]
Minimizing $\bar{D}(K,B)$ is equivalent to maximizing $\E[\sum_{i\in\cB} T_i]$.
Theorem~\ref{thm:optimal-hitrate} gives
\[
\E\Big[\sum_{i\in\cB} T_i\Big]
\le
\E\Big[\sum_{\ell=1}^{B} \eta_{(\ell)}\Big].
\]
Substituting into the previous display yields the converse $\bar{D}(K,B)\ge \bar{D}^\star(K,B)$ with $\bar{D}^\star$ defined in \eqref{eq:Dstar-finite}.
For achievability, verify the $B$ largest scores.
Then Theorem~\ref{thm:optimal-hitrate} holds with equality and the Bayes predictors achieve the stated risk.
\end{proof}

\subsection{Proof of Corollary~\ref{cor:single-letter}}
\label{app:proof-single-letter}
\begin{proof}
Let $\eta_1,\dots,\eta_K$ be i.i.d. samples of $\eta(Z)$ with a continuous distribution.
Fix $\alpha\in(0,1)$ and set $B=\lfloor \alpha K\rfloor$.
Define the population threshold $t_\alpha := Q_\eta(1-\alpha)$.
Continuity implies $\Prob(\eta=t_\alpha)=0$.
Standard order-statistics theory implies the empirical quantile converges almost surely:
\[
\eta_{(B)}\to t_\alpha,\qquad\text{as }K\to\infty.
\]
Since there are no ties almost surely,
\[
\frac{1}{K}\sum_{\ell=1}^{B} \eta_{(\ell)}
=
\frac{1}{K}\sum_{i=1}^{K} \eta_i\,\1\{\eta_i\ge \eta_{(B)}\}.
\]
The difference between this term and $\frac{1}{K}\sum_{i=1}^K \eta_i\,\1\{\eta_i\ge t_\alpha\}$ is controlled by the empirical measure of the interval between $\eta_{(B)}$ and $t_\alpha$.
This measure converges to $0$ almost surely.
Therefore,
\[
\frac{1}{K}\sum_{\ell=1}^{B} \eta_{(\ell)}
\to
\E\bracks{\eta\,\1\{\eta\ge t_\alpha\}}
=
\int_{1-\alpha}^{1} Q_\eta(u)\,du.
\]
This proves \eqref{eq:topB-quantile}.
Substituting into \eqref{eq:Dstar-finite} yields \eqref{eq:Dstar-single-letter}.
\end{proof}

\end{appendix}

%% file: sections/tight_region.tex
\section{A Fully Solved Benchmark: A Tight Risk-Resource Region}
\label{app:tight-benchmark}

Section~\ref{sec:theory} derives general converse bounds under budgets $(K,B)$.
It also gives an achievability result for a decoupled claim model.
In this appendix we complement that analysis with a benchmark that admits an exact characterization.
The benchmark is relevant when each informative record carries its own independent latent claim.
This decoupling removes global coupling and turns the problem into an optimal selection task.

\subsection{A canonical haystack model}
We consider a large pool of candidate records.
Each record has an unobserved binary type $T\in\{0,1\}$, where $T=1$ means the record is informative and $T=0$ means the record is uninformative.
We write $p:=\Prob(T=1)$.

Each informative record is associated with an independent latent claim $\Theta\in\cT$ with prior $P_\Theta$.
Different informative records have independent latent claims, and all latent claims share the same prior.

An inspection reveals a screening statistic $Z\in\cZ$.
The screening channel is specified by $P_{Z\mid T=1}$ and $P_{Z\mid T=0}$.
After observing $Z=z$, the posterior probability that the record is informative is
\begin{equation}
  \eta(z) := \Prob(T=1\mid Z=z).
  \label{eq:eta-def}
\end{equation}
We call $\eta(Z)$ the score.
The score is a sufficient statistic for selecting which inspected records to verify.

A verification reveals a pair $(T,V)$, where $V\in\cV$ is a verification output.
We assume the following verification model.

\begin{assumption}[Verification as a side channel]
\label{asmp:verification-sidechannel}
If $T=0$, then $V$ is a fixed symbol $v_0$ and is independent of $\Theta$.
If $T=1$, then $V$ is generated according to a channel $P_{V\mid \Theta}$.
\end{assumption}

We will return to this benchmark at the end of the appendix with a finite-length simulation that validates the resulting boundary and its weak-screening approximation (Fig.~\ref{fig:finite-length-validation}).

Under log-loss, the Bayes risk for an informative verified record equals $\Hh(\Theta\mid V)$.
For an uninformative verified record it equals $\Hh(\Theta)$.
It is convenient to define the information contribution of verification as
\begin{equation}
  I_{\mathrm{ver}}^{\mathrm{bm}} := \I(\Theta;V),
  \label{eq:Iver-bm}
\end{equation}
where the mutual information is computed under $P_\Theta P_{V\mid\Theta}$.
In this benchmark, $I_{\mathrm{ver}}^{\mathrm{bm}}$ coincides with $I_{\mathrm{ver}}$ from \eqref{eq:iver}.

In one window, we inspect $K$ records and then verify exactly $B$ of the inspected records, where $1\le B\le K$.
The inspection outcomes are i.i.d.\ samples $(Z_1,\dots,Z_K)$ drawn from the mixture distribution induced by $T$ and $P_{Z\mid T}$.

A verification policy maps the scores $\eta_1,\dots,\eta_K$ to a verification set $\cB\subseteq\idx{K}$ with $\abs{\cB}=B$.
We allow randomized policies.

\subsection{Benchmark risk-resource region}
We evaluate performance by the average log-loss over the $B$ verified records.
Let $q_i(\cdot)$ denote the predicted distribution for the latent claim of verified record $i$.
The benchmark risk of a policy is
\begin{equation}
  \bar{D}(K,B) := \frac{1}{B}\sum_{i\in\cB} \E\bracks{-\log q_i(\Theta_i)},
  \label{eq:risk-verified-bar}
\end{equation}
where $\Theta_i$ denotes the latent claim associated with record $i$ when $T_i=1$.

We say that a triple $(K,B,D)$ is achievable if there exists a policy using at most $K$ inspections and $B$ verifications such that $\bar{D}(K,B)\le D$.
The benchmark risk-resource region is the set of all achievable triples.

The next theorem characterizes the optimal expected number of informative verifications.

\begin{theorem}[Optimal hit rate under score-based selection]
\label{thm:optimal-hitrate}
Let $\eta_i:=\eta(Z_i)$ be the score of inspected record $i$, and let $\eta_{(1)}\ge \eta_{(2)}\ge \cdots \ge \eta_{(K)}$ denote the order statistics.
For any verification policy,
\begin{equation}
  \E\bracks{\sum_{i\in\cB} T_i}
  \le
  \E\bracks{\sum_{\ell=1}^{B} \eta_{(\ell)}}.
  \label{eq:hit-rate-opt}
\end{equation}
Moreover, the upper bound is achieved by verifying the $B$ records with the largest scores.
\end{theorem}

\begin{figure*}[t]
\centering
\includegraphics[width=\textwidth]{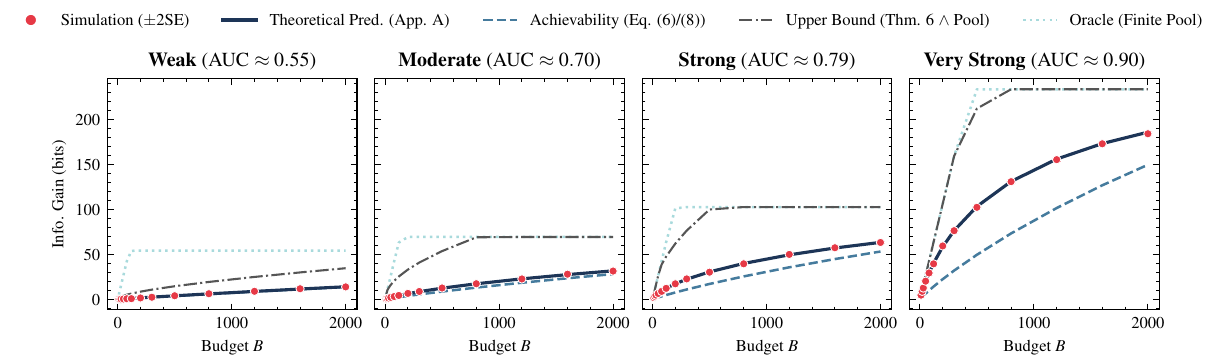}
\caption{\textbf{Finite-length validation of the JaKoB scaling law.}
We simulate screening scores from a logistic model and apply the top-$B$ verification policy with $K=10^4$, $\delta=0.1$, and $p_0=0.01$.
Markers (with $\pm2$SE error bars) report the empirical information gain (log-loss reduction, in bits).
The solid curve is the benchmark prediction from Theorem~\ref{thm:tight-region} (Eq.~\eqref{eq:Dstar-finite}); the dashed curve is the weak-screening approximation from Theorem~\ref{thm:score-achievability}/Corollary~\ref{cor:tight-sqrt}.
Dash-dot curves show the converse upper bound from Theorem~\ref{thm:ver-tradeoff} intersected with the finite-pool constraint, and the dotted curve indicates the finite-pool oracle.
Across four screening strengths (AUC $\approx 0.55, 0.70, 0.79, 0.90$), the empirical gains closely track the benchmark and remain below the theoretical and oracle ceilings, while the weak-screening law becomes conservative as screening strengthens. The code to reproduce this plot is available at \faGithub~{\small \url{https://github.com/youlei202/Attention-Constrained-Inference}}.}
\label{fig:finite-length-validation}
\end{figure*}

Theorem~\ref{thm:optimal-hitrate} reduces the verification stage to an optimal top-$B$ selection problem once we condition on inspection outcomes.
The boundary of the region therefore depends on a single functional of the score distribution.

\begin{theorem}[Tight tradeoff between risk and resource]
\label{thm:tight-region}
Under Assumption~\ref{asmp:verification-sidechannel}, the minimum achievable risk under $(K,B)$ equals
\begin{equation}
  \bar{D}^{\star}(K,B)
  =
  \Hh(\Theta)
  -
  \frac{I_{\mathrm{ver}}^{\mathrm{bm}}}{B}\,\E\bracks{\sum_{\ell=1}^{B} \eta_{(\ell)}}.
  \label{eq:Dstar-finite}
\end{equation}
Equivalently, a triple $(K,B,D)$ is achievable if and only if $D\ge \bar{D}^{\star}(K,B)$.
\end{theorem}

Theorem~\ref{thm:tight-region} gives a complete benchmark characterization.
It also makes clear why upper-tail properties matter.
The boundary is expressed through the expected sum of the top-$B$ posterior scores among the $K$ inspected records.

\subsection{An asymptotic single-letter expression}
Equation~\eqref{eq:Dstar-finite} is exact for any finite $(K,B)$.
In the large-system regime the boundary admits a single-letter form.

Let $F_\eta$ be the cumulative distribution function of the score $\eta(Z)$, and let $Q_\eta(u):=\inf\{x: F_\eta(x)\ge u\}$ be the quantile function.
Fix a verification fraction $\alpha\in(0,1)$ and set $B=\lfloor \alpha K\rfloor$.

\begin{corollary}[Single-letter boundary for $K\to\infty$]
\label{cor:single-letter}
Assume $\eta(Z)$ has a continuous distribution.
As $K\to\infty$ with $B=\lfloor \alpha K\rfloor$,
\begin{equation}
  \frac{1}{K}\,\E\bracks{\sum_{\ell=1}^{B} \eta_{(\ell)}}
  \to
  \int_{1-\alpha}^{1} Q_\eta(u)\,du.
  \label{eq:topB-quantile}
\end{equation}
Consequently,
\begin{equation}
  \bar{D}^{\star}(K,\lfloor \alpha K\rfloor)
  =
  \Hh(\Theta)
  -
  \frac{I_{\mathrm{ver}}^{\mathrm{bm}}}{\alpha}\,\int_{1-\alpha}^{1} Q_\eta(u)\,du
  + o(1).
  \label{eq:Dstar-single-letter}
\end{equation}
\end{corollary}

Corollary~\ref{cor:single-letter} turns the benchmark region into a single-letter expression.
The integral in \eqref{eq:Dstar-single-letter} is an upper-tail mean of the score distribution.
It is a canonical tail-leverage quantity.

\subsection{Finite-length validation (simulation)}
\label{subsec:finite-length-validation}

To validate that the benchmark characterization remains predictive away from the asymptotic regime, we simulate a finite pool of $K=10^4$ inspected records under a logistic score model consistent with Assumption~\ref{asmp:weak-screening}. Concretely, we generate posterior scores of the form $\eta_i=\sigma(\mathrm{logit}(p_0)+\delta G_i)$, where $\sigma$ is the logistic function, $p_0=0.01$ is the base rate, $\delta=0.1$ controls screening strength, and $G_i$ is standardized. We then apply the top-$B$ verification policy and report the empirical information gain $\Hh(\Theta)-D(K,B)$.

Fig.~\ref{fig:finite-length-validation} compares these empirical gains to three theoretical references: (i) the benchmark prediction from Theorem~\ref{thm:tight-region} via \eqref{eq:Dstar-finite} (solid), (ii) the weak-screening approximation implied by Theorem~\ref{thm:score-achievability} / Corollary~\ref{cor:tight-sqrt} (dashed), and (iii) the converse bound from Theorem~\ref{thm:ver-tradeoff} intersected with the finite-pool ceiling (dash-dot). The dotted curve is a finite-pool oracle that knows which records are informative and can therefore verify only true positives up to the pool limit.